\documentclass{article} 

\usepackage[dvipsnames]{xcolor}

\usepackage[final,nonatbib]{neurips_2022}

\usepackage{amsmath, amssymb, amsthm}
\usepackage{graphicx}
\usepackage{bbm}

\usepackage{hyperref}

\usepackage[makeroom]{cancel}

\usepackage{wrapfig}

\usepackage{pgf, tikz}
\usetikzlibrary{arrows, automata, positioning}

\newcommand{\R}{\mathbb{R}}
\newcommand{\1}{\mathbbm{1}}
\newcommand{\E}{\mathbb{E}}

\usepackage{accents}

\usepackage{xspace}
\newcommand{\rvs}{\text{RCSL}\xspace}

\usepackage[ruled, linesnumbered]{algorithm2e}

\newtheorem{theorem}{Theorem}

\newtheorem{corollary}{Corollary}
\newtheorem{lemma}{Lemma}

\usepackage{caption}
\usepackage{subcaption}

\usepackage{enumitem}

\usepackage[square,numbers]{natbib}
\definecolor{color1}{RGB}{215,25,28}
\definecolor{color2}{RGB}{26,150,65}
\definecolor{color3}{RGB}{200,10,10}

\title{When does return-conditioned supervised learning work for offline reinforcement learning?}

\author{David Brandfonbrener\\
        New York University\\
        \texttt{david.brandfonbrener@nyu.edu} \And Alberto Bietti\\
        New York University
        \And Jacob Buckman \\
        MILA
        \AND Romain Laroche \\
        Microsoft Research
        \And Joan Bruna \\
        New York University}

\begin{document}

\maketitle

\begin{abstract}
    Several recent works have proposed a class of algorithms for the offline reinforcement learning (RL) problem that we will refer to as return-conditioned supervised learning (\rvs). \rvs algorithms learn the distribution of actions conditioned on both the state and the return of the trajectory. Then they define a policy by conditioning on achieving high return. In this paper, we provide a rigorous study of the capabilities and limitations of \rvs, something which is crucially missing in previous work. We find that \rvs returns the optimal policy under a set of assumptions that are stronger than those needed for the more traditional dynamic programming-based algorithms. We provide specific examples of MDPs and datasets that illustrate the necessity of these assumptions and the limits of \rvs. Finally, we present empirical evidence that these limitations will also cause issues in practice by providing illustrative experiments in simple point-mass environments and on datasets from the D4RL benchmark.  
\end{abstract}

\section{Introduction}

In recent years, deep learning has proven to be an exceptionally powerful generic algorithm for solving supervised learning (SL) tasks. These approaches tend to be stable, and scale well with compute and data \citep{kaplan2020scaling}. In contrast, deep reinforcement learning algorithms seem to lack these nice properties; results are well known to be sensitive to hyperparameters and difficult to replicate. In spite of this, deep reinforcement learning (RL) has achieved impressive feats, such as defeating human champions at Go \citep{Silver2016MasteringTG}. This juxtaposition of success and instability has inspired researchers to explore alternative approaches to reinforcement learning that more closely resemble supervised learning in hopes of making deep RL as well-behaved as deep SL.

One family of algorithms that has garnered great interest recently is return-conditioned supervised learning (\rvs). The core idea of \rvs is to learn the return-conditional distribution of actions in each state, and then define a policy by sampling from the distribution of actions that receive high return. This was first proposed for the online RL setting by work on Upside Down RL \citep{schmidhuber2019reinforcement, srivastava2019training} and Reward Conditioned Policies \citep{kumar2019reward}. The idea was extended to the offline RL setting using transformers that condition on the entire history of states rather than just the current Markovian state in the Decision Transformer (DT) work \citep{chen2021decision, furuta2021generalized}. Recent work on RL via Supervised Learning (RvS) \citep{emmons2021rvs} unifies and simplifies ideas from these prior works with ideas about goal-conditioned policies. 

Importantly, none of this prior work provides theoretical guarantees or analysis of the failure modes of the return-conditioning approach. In contrast, the more established dynamic programming (DP) algorithms for RL are better understood theoretically. This paper attempts to address this gap in understanding, in order to assess when \rvs is a reliable approach for offline RL. Specifically, we answer the following questions:
\begin{itemize}
    \item What optimality guarantees can we make for \rvs? Under what conditions are they necessary and sufficient?
    \item In what situations does \rvs fail in theory and in practice?
    \item How does \rvs relate to other approaches, such as DP and behavior cloning (BC)? 
\end{itemize}

We find that although \rvs does select a near-optimal policy under certain conditions, the necessary assumptions are more strict than those for DP. In particular, \rvs (but not DP) requires nearly deterministic dynamics in the MDP, knowledge of the proper value to condition on, and for the conditioning value to be supported by the distribution of returns in the dataset. We provide simple tabular examples to demonstrate the necessity of these assumptions.
The shortcomings of \rvs that we identify in theory are verified empirically with some simple experiments using neural models on ad-hoc example problems as well as benchmark datasets.
We conclude that \rvs alone is unlikely to be a general solution for offline RL problems, but does show promise in some specific situations such as deterministic MDPs with high-quality behavior data.

\section{Preliminaries}

\subsection{Setup}\label{sec:setup}

We will consider an offline RL setup where we are given a dataset $ \mathcal{D}$ of trajectories $ \tau = (o_1, a_1, r_1, \cdots, o_H, a_H, r_H)$ of observations $ o_t \in \mathcal{O}$, actions $ a_t \in \mathcal{A}$, and rewards $ r_t \in [0,1]$ generated by some behavior policy $ \beta$ interacting with a finite horizon MDP with horizon $ H$. Let $ g(\tau) = \sum_{t=1}^H r_t$ denote the cumulative return of the trajectory (we will just use $ g $ when the trajectory is clear from context). And let $ J(\pi) = \E_{\tau \sim \pi}[g(\tau)]$ be the expected return of a policy $ \pi$.
We then let the state representation $ s_t \in \mathcal{S}$ be any function of the history of observations, actions, and rewards up to step $ t $ along with $ o_t$.
To simplify notation in the finite horizon setting, we will sometimes drop the timestep from $ s $ to refer to generic states and assume that we can access the timestep from the state representation as $ t(s)$. 
Let $ P_\pi$ denote the joint distribution over states, actions, rewards, and returns induced by any policy $ \pi$. 

In this paper, we focus on the \rvs approach that learns by return-conditioned supervised learning. Explicitly, at training time this method minimizes the empirical negative log likelihood loss:
\begin{align}
    \hat L(\pi) = - \sum_{\tau \in \mathcal{D}} \sum_{1 \leq t\leq H}  \log \pi(a_t|s_t, g(\tau)).
\end{align}
Then at test time, an algorithm takes the learned policy $ \pi$ along with a conditioning function $ f(s) $ to define the test-time policy $ \pi_f $ as:
\begin{align}
    \pi_f(a|s) := \pi(a|s, f(s)).
\end{align}
Nota bene: the Decision Transformer \citep{chen2021decision} is captured in this framework by defining the state space so that the state $ s_t$ at time $ t$ also contains all past $ o_{t'}$, $ a_{t'}$, and $ r_{t'}$ for $ t'< t$. In prior work, $ f $ is usually chosen to be a constant at the initial state and to decrease with observed reward along a trajectory, which is captured by a state representation that includes the history of rewards.

\subsection{The \rvs policy}\label{sec:policy}

To better understand the objective, it is useful to first consider its optimum in the case of infinite data. It is clear that our loss function attempts to learn $ P_\beta(a|s,g)$ where $\beta$ is the behavior policy that generated the data (and recall that $ P_\beta$ refers to the distribution over states, actions, and returns induced by $ \beta$). Factoring this distribution, we quickly see that the optimal policy $ \pi^{\rvs}_f$ for a specific conditioning function $ f $ can be written as:
\begin{align}
    \pi^{\rvs}_f(a|s) = P_\beta(a|s,f(s)) = \frac{P_\beta(a|s) P_\beta(f(s)|s,a)}{P_\beta(f(s)|s)} = \beta(a|s) \frac{P_\beta(f(s)|s,a)}{P_\beta(f(s)|s)}.
\end{align}
Essentially, the \rvs policy re-weights the behavior based on the distribution of future returns.

\paragraph{Connection to distributional RL.} In distributional RL \citep{bdr2022}, the distribution of future returns under a policy $ \pi$ from state $ s$ and action $ a $ is defined as: $G^\pi(s,a) \sim g = \sum_{t=t(s)}^H r_{t} \;\;|\;\; \tau \sim \pi, s_{t(s)} = s, a_{t(s)} = a$.
The \rvs policy is precisely proportional to the product of the behavior policy and the density of the distributional Q function of the behavior policy (i.e. $ P_\beta(g|s,a)$).

\subsection{Related work}

As noted in the introduction, our work is in direct response to the recent line of literature on \rvs \citep{schmidhuber2019reinforcement, srivastava2019training, kumar2019reward, chen2021decision, furuta2021generalized, emmons2021rvs}. Specifically, we will focus on the DT \citep{chen2021decision} and RvS \citep{emmons2021rvs} formulations in our experiments since they also focus on the offline RL setting. Note that another recent work introduced the Trajectory Transformer \citep{janner2021offline} which does not fall under the \rvs umbrella since it performs planning in the learned model to define a policy.

Another relevant predecessor of \rvs comes from work on goal-based RL \citep{Kaelbling1993LearningTA}. Compared to \rvs, this line of work replaces the target return $ g $ in the empirical loss function by a goal state. One instantiation is hindsight experience replay (HER) where each trajectory in the replay buffer is relabeled as if the observed final state was in fact the goal state \citep{andrychowicz2017hindsight}. 
Another instance is goal-conditioned supervised learning \cite[GCSL,][]{ghosh2019learning}, which provides more careful analysis and guarantees, but the guarantees (1) are not transferable to the return-conditioned setting, (2) assume bounds on $ L_\infty$ errors in TV distance instead of dealing with expected loss functions that can be estimated from data, and (3) do not provide analysis of the tightness of the bounds.

Concurrent work \cite{strupl2022upsidedown, paster2022you, yang2022dichotomy} also all raise the issue of RCSL in stochastic environments with infinite data, and present some algorithmic solutions. However, none of this work addresses the potentially more fundamental issue of sample complexity that arises from the requirement of return coverage that we discuss in Section 4.

\section{When does \rvs find the optimal policy?}

We begin by exploring how \rvs behaves with infinite data and a fully expressive policy class. In this setting, classic DP algorithms (e.g. Q-learning) are guaranteed to converge to the optimal policy under coverage assumptions \cite{sutton2018reinforcement}. But we now show that this is not the case for \rvs, which requires additional assumptions for a similar guarantee.
Our approach is to first derive a positive result: under certain assumptions, the policy which optimizes the \rvs objective (Section \ref{sec:policy}) is guaranteed to be near-optimal. We then illustrate the limitations of \rvs by providing simple examples that are nonetheless challenging for these methods in order to demonstrate why our assumptions are necessary and that our bound is tight.


\begin{theorem}[Alignment with respect to the conditioning function]\label{thm:infinite}
Consider an MDP, behavior $ \beta$ and conditioning function $ f$. Assume the following:
\begin{enumerate}
    \item Return coverage: $ P_\beta(g=f(s_1)|s_1) \geq \alpha_f$ for all initial states $ s_1$.
    \item Near determinism: $ P(r \neq r(s, a) \text{ or } s' \neq T(s, a) | s,a ) \leq \epsilon$ at all $ s, a $ for some functions $ T$ and $ r $. Note that this does not constrain the stochasticity of the initial state.
    \item Consistency of $ f$: $ f(s) = f(s') + r$ for all $ s$.\footnote{Note this can be exactly enforced (as in prior work) by augmenting the state space to include the cumulative reward observed so far.}
\end{enumerate}
Then
\begin{align}
    \E_{s_1}[f(s_1)] - J(\pi_f^{\rvs}) \leq \epsilon\left( \frac{1}{\alpha_f} + 2\right)H^2.
\end{align}
Moreover, there exist problems where the bound is tight up to constant factors.
\end{theorem}

The proof is in Appendix \ref{app:infinite}. Note that the quantity $ \E_{s_1}[f(s_1)]$ is specific to the structure of \rvs algorithms and captures the notion that the ideal \rvs policy will be able to reproduce policies of any value when given different conditioning functions (with appropriate data). The theorem immediately yields the following corollaries (with proof in Appendix \ref{app:infinite}).

\begin{corollary}\label{cor:soft-infinite}
Under the assumptions of Theorem \ref{thm:infinite}, there exists a conditioning function $ f $ such that 
\begin{align}
    J(\pi^*) - J(\pi_f^{\rvs}) \leq \epsilon\left( \frac{1}{\alpha_f} + 3\right)H^2.
\end{align}
\end{corollary}

\begin{corollary}\label{cor:hard-infinite}
If $ \alpha_f > 0$, $ \epsilon = 0$, and $ f(s_1) = V^*(s_1)$ for all initial states $s_1$, then $ J(\pi_f^{\rvs}) = J(\pi^*)$.
\end{corollary}

The corollaries tell us that in near determinisitc environments with the proper conditioning functions and data coverage, it is possible for \rvs to recover near optimal policies. These assumptions are somewhat strong compared to those needed for DP-based approaches, so we will now explain why they are necessary for our analysis. 


\paragraph{Tightness.} To demonstrate tightness we will consider the simple examples in Figure \ref{fig:toy1}. These MDPs and behavior policies demonstrate tightness in $ \epsilon$ and $ \alpha_f$ up to constant factors, and provide insight into how stochastic dynamics lead to suboptimal behavior from \rvs algorithms. 

\begin{figure}[h]
    \centering
    \begin{subfigure}[t]{0.32\textwidth}
       \centering
       \includegraphics[width=0.8\textwidth]{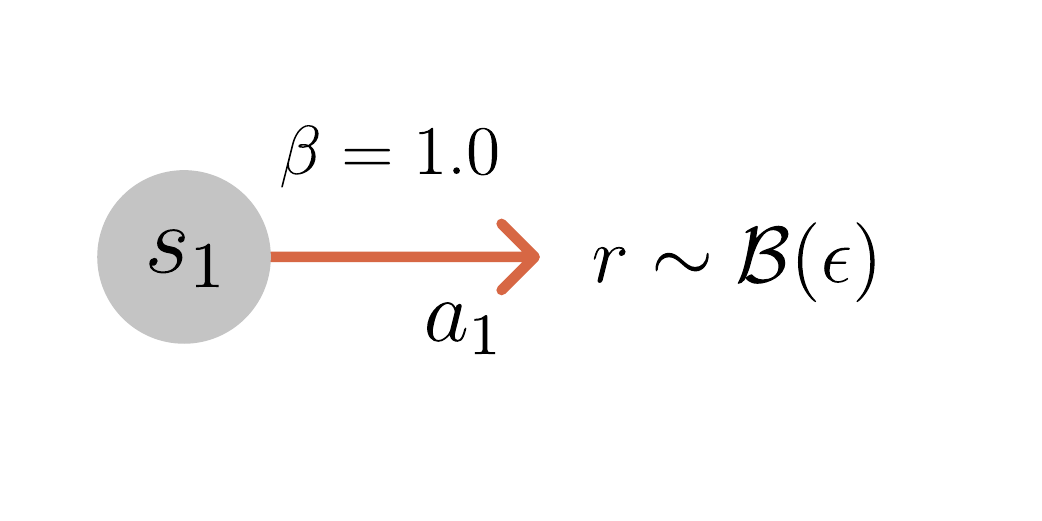}
        \caption{An example where the bound is tight. $ \mathcal{B}$ denotes the Bernoulli distribution.}
        \label{fig:onearm}
    \end{subfigure}
    \ \ 
        \begin{subfigure}[t]{0.32\textwidth}
       \centering
       \includegraphics[width=0.8\textwidth]{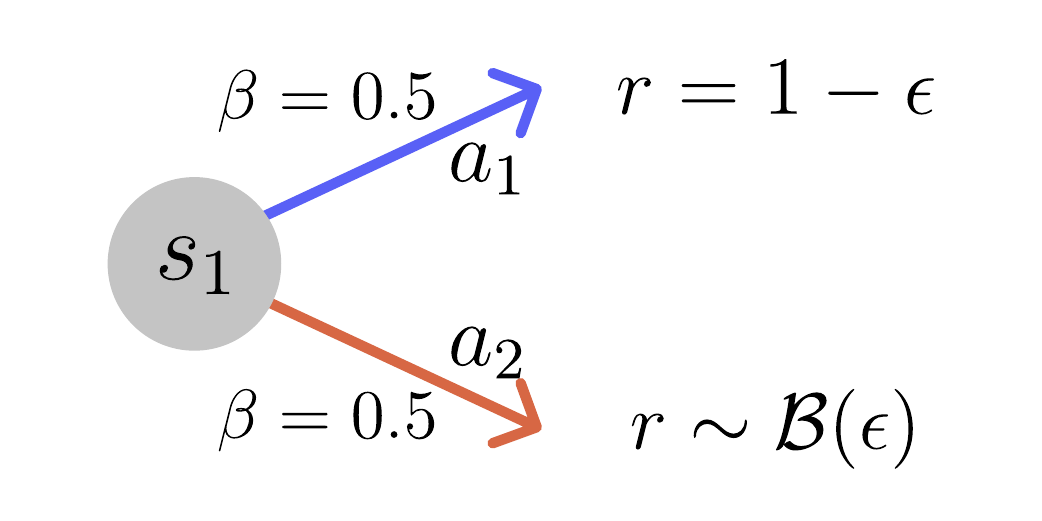}
        \caption{An example where \rvs also has large regret.}
        \label{fig:max}
    \end{subfigure}
    \ \ 
    \begin{subfigure}[t]{0.32\textwidth}
       \centering
       \includegraphics[width=0.8\textwidth]{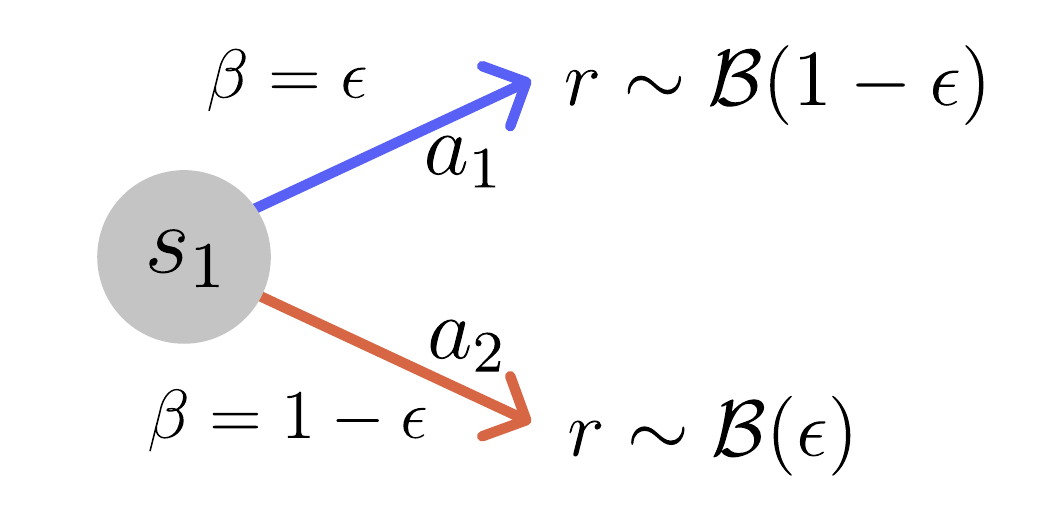}
        \caption{An example where \rvs also has large regret for any conditioning function.}
        \label{fig:bias}
    \end{subfigure}
    \caption{Failure modes of \rvs in stochastic environments with infinite data.}
    \label{fig:toy1}
\end{figure}

First, consider the example in Figure \ref{fig:onearm} with conditioning $ f(s_1) = 1$. There is only one possible policy in this case, and it has $ J(\pi) = \epsilon$ so that  $  \E[f(s_1)] - J(\pi) = 1 - \epsilon$. Note that $ \alpha_f = \epsilon$, so we have that $ \epsilon / \alpha_f = 1$. Thus, the bound is tight in $ \epsilon / \alpha_f$. This example shows that the goal of achieving a specific desired return is incompatible with stochastic environments.

This first example is somewhat silly since there is only one action, so the learned policy does not actually suffer any regret. To show that this issue can in fact lead to regret, consider the example in Figure \ref{fig:max}, again with conditioning $ f(s_1) = 1$. Then applying the reasoning from Section \ref{sec:policy},
\begin{align}
    \pi_f^\rvs({\color{blue}a_1}|s_1) &= \beta({\color{blue}a_1}|s_1) \frac{P_\beta(g=1|s_1, {\color{blue}a_1})}{P_\beta(g=1|s_1)} = 0.5 \cdot \frac{0}{0.5 \cdot \epsilon} = 0.
\end{align}
So we get that $ \E[f(s_1)] - J(\pi_f^\rvs) = 1 - \epsilon$, while $ \epsilon / \alpha_f  = \epsilon / (\epsilon / 2) = 2$ (which is on the same order, up to a constant factor). However, in this case the learned policy $ \pi_f^\rvs$ suffers substantial regret since the chosen action $ {\color{red}a_2}$ has substantially lower expected value than $ {\color{blue}a_1}$ by $ 1 - 2 \epsilon$.

The issue in the second example could be resolved by changing the conditioning function so that $ f(s_1) = 1 - \epsilon$. Now we will consider the example in Figure \ref{fig:bias} where we will see that there exist cases where the bias of \rvs in stochastic environments can remain regardless of the conditioning function. In this MDP, the only returns that are supported are $ g = 0$ or $ g = 1$. For $ f(s_1) = 1$, plugging in to the formula for $ \pi_f$ yields
\begin{align}
    \pi_f^\rvs({\color{blue}a_1}|s_1) = \beta({\color{blue}a_1}|s_1) \frac{P_\beta(g=1|s_1, {\color{blue}a_1})}{P_\beta(g=1|s_1)} = \epsilon \frac{1 - \epsilon}{\epsilon (1-\epsilon) + (1-\epsilon) \epsilon} = \frac{1}{2}.
\end{align}
Thus, $ \E[f(s_1)] - J(\pi_f^\rvs) = 1/2 $ and $ J(\pi^*) - J(\pi_f^\rvs) = 1/2 - \epsilon$. This shows that merely changing the conditioning function is not enough to overcome the bias of the \rvs method in stochastic environments.

These examples show that even for MDPs that are $ \epsilon$-close to being deterministic, the regret of \rvs can be large. But, in the special case of deterministic MDPs we find that \rvs can indeed recover the optimal policy. And note that we still allow for stochasticity in the initial states in these deterministic MDPs, which provides a rich setting for problems like robotics that requires generalization over the state space from finite data. In the next section, we will consider more precisely what happens to \rvs algorithms with finite data and limited model classes.

\paragraph{Trajectory stitching.} Another issue often discussed in the offline RL literature is the idea of trajectory stitching \citep{wang2020critic, chen2021decision}. Ideally, an offline RL agent can take suboptimal trajectories that overlap and stitch them into a better policy. Clearly, DP-based algorithms can do this, but it is not so clear that \rvs algorithms can. In Appendix \ref{app:stitch} we provide theoretical and empirical evidence that in fact they cannot perform stitching in general, even with infinite data. While this does not directly affect our bounds, the failure to perform stitching is an issue of practical importance for \rvs methods.

\section{Sample complexity of \rvs}\label{sec:finite}

Now that we have a better sense of what policy \rvs will converge to with infinite data, we can consider how quickly (and under what conditions) it will converge to the policy $ \pi_f$ when given finite data and a limited policy class, as will occur in practice.
We will do this via a reduction from the regret relative to the infinite data solution $ \pi_f$ to the expected loss function $L $ minimized at training time by \rvs, which is encoded in the following theorem.

\begin{theorem}[Reduction of \rvs to SL]\label{thm:finite}
Consider any function $ f: \mathcal{S} \to \R$ such that the following two assumptions hold: 
\begin{enumerate}
    \item Bounded occupancy mismatch: $\frac{P_{\pi_f^\rvs}(s)}{P_\beta(s)} \leq C_f$ for all $ s$.
    \item Return coverage: $P_{\beta}(g = f(s)|s) \geq \alpha_f $ for all $ s$.
\end{enumerate}
Define the expected loss as $L(\hat \pi) = \E_{s \sim P_\beta}\E_{g \sim P_\beta(\cdot|s)} [D_\mathrm{KL}(P_\beta(\cdot|s,g)\| \hat \pi(\cdot |s,g))]$. Then for any estimated \rvs policy $ \hat \pi $ that conditions on $ f $ at test time (denoted by $ \hat \pi_f$), we have that 
\begin{align}
    J(\pi_f^{\rvs}) - J(\hat \pi_f) \leq  \frac{C_f}{\alpha_f} H^2 \sqrt{2 L(\hat \pi)}. 
\end{align}
\end{theorem}
The proof can be found in Appendix \ref{app:finite}. 
Note that we require a similar assumption of return coverage as before to ensure we have sufficient data to define $ \pi_f$. We also require an assumption on the state occupancy of the idealized policy $ \pi_f$ relative to $ \beta$. This assumption is needed since the loss $ L(\hat \pi)$ is optimized on states sampled from $ P_\beta$, but we care about the expected return of the learned policy relative to that of $ \pi_f$, which can be written as an expectation over states sampled from $ P_{\pi_f}$. 

This gives us a reduction to supervised learning, but to take this all the way to a sample complexity bound we need to control the loss~$L(\hat \pi)$ from finite samples. Letting $ N $ denote the size of the dataset, the following corollary uses standard uniform convergence results from supervised learning~\cite{shalev2014understanding} to yield finite sample bounds.


\begin{corollary}[Sample complexity of \rvs]\label{cor:finite}
To get finite data guarantees, add to the above assumptions the assumptions that (1) the policy class $ \Pi$ is finite, (2) $|\log \pi(a|s,g) - \log \pi(a'|s', g')| \leq c$ for any~$(a, s, g, a', s', g')$ and all~$\pi \in \Pi$, and (3) the approximation error of $ \Pi$ is bounded by $ \epsilon_{approx}$, i.e. $ \min_{\pi\in \Pi}L(\pi) \leq \epsilon_{approx}$. Then with probability at least $ 1-\delta$,
\begin{align}
    J(\pi_f^{\rvs}) - J(\hat \pi_f) \leq O \left ( \frac{C_f}{\alpha_f}H^2\left( \sqrt{c}\left(\frac{\log |\Pi| /\delta}{N}\right)^{1/4} + \sqrt{\epsilon_{approx}} \right)\right). 
\end{align}
\end{corollary}

The proof is in Appendix \ref{cor:finite_proof}.
Analyzing the bound, we can see that the dependence on $ N $ is in terms of a fourth root rather than the square root, but this comes from the fact that we are optimizing a surrogate loss. Namely the learner optimizes KL divergence, but we ultimately care about regret which we access by using the KL term to bound a TV divergence and thus lose a square root factor.
A similar rate appears, for example, when bounding 0-1 classification loss of logistic regression~\cite{bartlett2006convexity,boucheron2005theory}.

This corollary also tells us something about how the learner will learn to generalize across different values of the return. If the policy class is small (for some notion of model complexity) and sufficiently structured, then it can use information from the given data to generalize across values of $ g$, using low-return trajectories to inform the model on high-return trajectories. 

Note that a full sample complexity bound that competes with the optimal policy can be derived by combining this result with Corollary \ref{cor:soft-infinite} as follows:
\begin{corollary}[Sample complexity against the optimal policy]\label{cor:combined}
Under all of the assumptions of Corollary \ref{cor:soft-infinite} and Corollary \ref{cor:finite} we get:
\begin{align}
    J(\pi^*) - J(\hat \pi_f) \leq O \left ( \frac{C_f}{\alpha_f}H^2\left( \sqrt{c} \left(\frac{\log |\Pi| /\delta}{N}\right)^{1/4} + \sqrt{\epsilon_{approx}} \right) + \frac{\epsilon}{\alpha_f} H^2\right). 
\end{align}
\end{corollary}

\paragraph{Tightness.} To better understand why the dependence on $ 1/\alpha_f$ is tight and potentially exponential in the horizon $ H$, even in deterministic environments, we offer the example in Figure \ref{fig:loop}. Specifically, we claim that any value of $ f(s_1)$ where the policy $ \pi_f^\rvs$ prefers the good action $ {\color{blue}a_1} $ from $ s_1$ will require on the order of $ 10^{H/2}$ samples in expectation to recover as $ \hat \pi_f$\footnote{Except for $ f(s_1) = 0$, which will yield a policy substantially worse than the behavior.}. 

\begin{wrapfigure}[15]{r}{0.4\textwidth}
       \centering
       \includegraphics[width=0.3\textwidth]{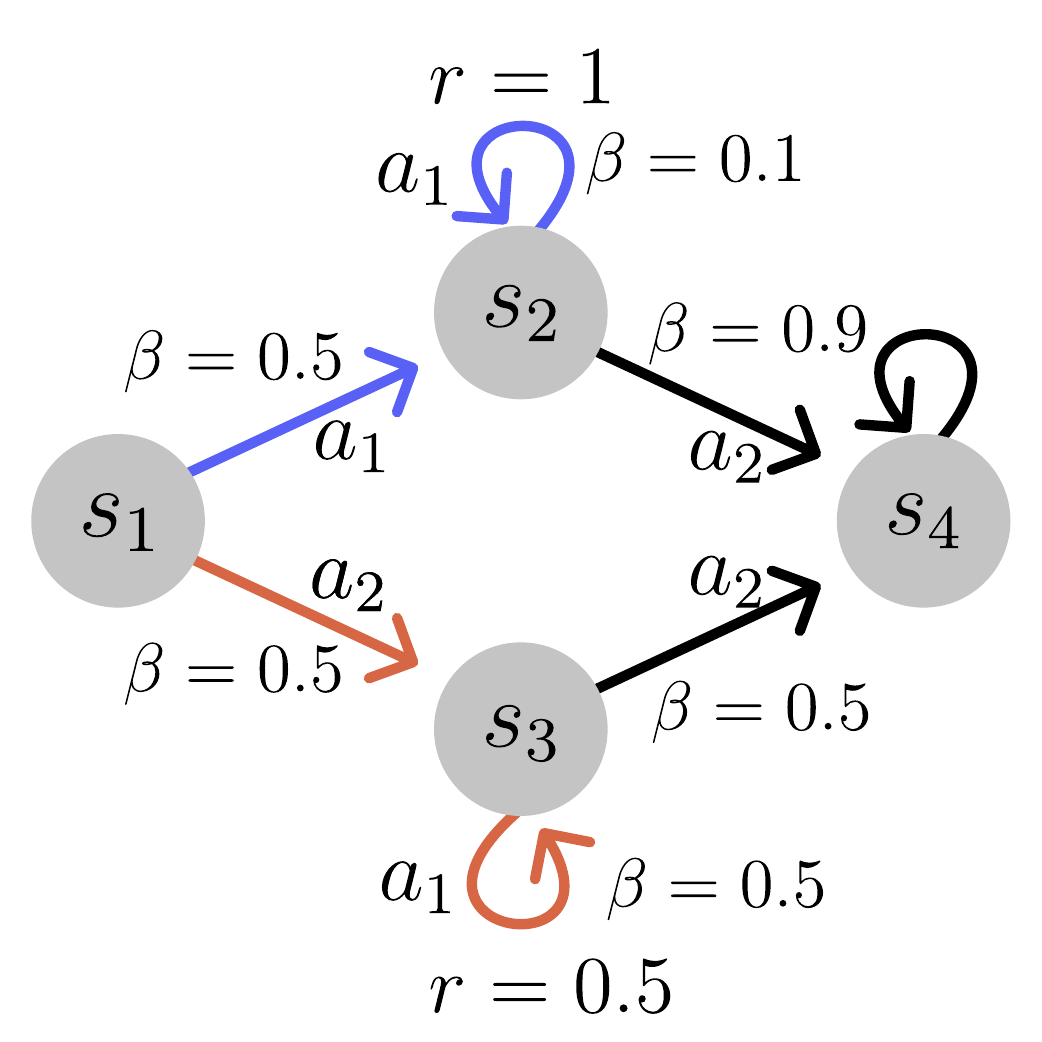}
        \caption{An example where \rvs has exponential sample complexity in a deterministic environment.}
        \label{fig:loop}
\end{wrapfigure}
To see why this is the case, we consider the MDP illustrated in Figure \ref{fig:loop} with horizon $ H\gg 4$. The MDP has four states each with two actions. All transitions and rewards are deterministic. The only actions with non-zero reward are $ r(s_2, {\color{blue}a_1}) = 1$ and $ r(s_3, {\color{red}a_1}) = 0.5$. The interesting decision is at $ s_1 $ where $ {\color{blue}a_1} $ is better than $ {\color{red}a_2}$.

Note that for any integer $ 1 \leq k < H/2$, we have that $ P_\beta(g = k|s_1, {\color{red}a_2}) = 0.5 \cdot 0.5^{2k} = 0.5 \cdot (0.25)^k$, while $ P_\beta(g = k|s_1, {\color{blue}a_1}) = 0.5 \cdot (0.1)^k$. Conditioning on any such $ k$ will make us more likely to choose the bad action ${\color{red}a_2}$ from $ s_1$. 
The only way to increase the likelihood of the good action $ {\color{blue}a_1}$ from $ s_1$ and $ s_2$ is to condition on $ f(s_1) > H/2$. Unfortunately for \rvs, the probability of observing $ g > H/2$ is extremely small, since for any such $ f $ we have $ P_\beta(g = f(s_1)) \leq 0.5 \cdot (0.1)^{H/2} \leq 10^{-H/2}$. Thus, both $ \alpha_f $ and the sample complexity of learning for any $ f $ that will yield a policy better than the behavior is exponential in the horizon $ H$.


Fundamentally, the problem here is that \rvs uses trajectory-level information instead of performing dynamic programming on individual transitions. But, collecting enough trajectory-level information can take exponentially many samples in the horizon. In contrast, DP merely requires coverage of transitions in the MDP to perform planning and thus avoids this issue of exponential sample complexity. In the next section we will delve deeper into this comparison with DP-based approaches as well as the simple top-\% BC baseline.

\section{Comparing \rvs with bounds for alternative methods}

Now that we understand the rate at which we expect \rvs to converge, we briefly present the convergence rates of two baseline methods for comparison. In particular, we will leverage an existing analysis of a DP-based algorithm, and conduct a novel analysis of top-\% BC. We find that the sample complexity of \rvs has a similar rate to top-\% BC, and is worse than DP due to the potentially exponential dependence on horizon that stems from return coverage.

\subsection{Comparison to dynamic programming.}

We will compare to the state of the art (to our knowledge) bound for a DP-based offline RL algorithm. Namely, we will look at the results of \cite{xie2021bellman} for pessimistic soft policy iteration. Similar results exist for slightly different algorithms or assumptions \citep{chen2019information, wang2020statistical}, but we choose this one since it is both the tightest and more closely aligns with the practical actor-critic algorithms that we use for our experiments. Their bound makes the following assumptions about the function class $ F $ and the dataset (letting $ \mathcal{T}^\pi $ represent the Bellman operator for policy $ \pi$): 
\begin{enumerate}[leftmargin=*]
    \item Realizability: for any policies $ \pi, \pi'$ there exists $ f \in F$ with $\|f - \mathcal{T}^\pi f\|_{2, P_{\pi'}}^2 \leq \epsilon_1$.
    \item Bellman completeness: for any $ \pi$ and $ f \in F$ there exists $ f' \in F$ such that  $\|f' - \mathcal{T}^\pi f\|_{2, P_{\beta}}^2 \leq \epsilon_2$.
    \item Coverage: $ \frac{P_{\pi^*}(s,a)}{P_\beta(s,a)} \leq C$ for all $ s,a$\footnote{The original paper uses a slightly tighter notion of coverage, but this bound will suffice for our comparison.}.
\end{enumerate}
With these assumptions in place, the sample complexity bound takes the form\footnote{The original paper considers an infinite horizon discounted setting. For the purposes of comparison, we will just assume that $ \frac{1}{1-\gamma}$ can be replaced by $ H$.}:
\begin{align}
    J(\pi^*) - J(\hat \pi) \leq O\left( H^2 \left(\sqrt{\frac{C \log |F||\Pi|/\delta}{N}} \right)  +  H^2\sqrt{C (\epsilon_1 + \epsilon_2)} \right)
\end{align}
Note: this is the result for the ``information-theoretic'' form of the algorithm that cannot be efficiently implemented. The paper also provides a ``practical'' version of the algorithm for which the bound is the same except that the the square root in the first term is replaced with a fifth root. 

There are several points of comparison with our analysis (specifically, our Corollary \ref{cor:combined}). The first thing to note is that for \rvs to compete with the optimal policy, we require nearly deterministic dynamics and a priori knowledge of the optimal conditioning function. These assumptions are not required for the DP-based algorithm; this is a critical difference, since it is clear that these conditions often do not hold in practice.

Comparing the coverage assumptions, our $ C_f$ becomes nearly equivalent to $ C$. The major difference is that our analysis of \rvs also requires dependence on return coverage $ 1/\alpha_f$. This is problematic since as seen in Section \ref{sec:finite}, this return coverage dependence can be exponential in horizon in cases where the state coverage does not depend on horizon.

Comparing the approximation error assumptions, we see that the realizability and completeness assumptions required for DP are substantially less intuitive than the standard supervised learning approximation error assumption needed for \rvs. These assumptions are not directly comparable, but intuitively the \rvs approximation error assumption is simpler. 

Finally, dependence on $ H $ is the same for both methods and dependence on $ N$ depends on which version of the DP algorithm we compare to. For the information-theoretic algorithm DP has better dependence on $ N$, but for the practical algorithm \rvs has better dependence. It is not clear whether the dependence on $ N$ in either the \rvs analysis or in the analysis of the practical algorithm from \cite{xie2021bellman} is tight, and it is an interesting direction for future work to resolve this issue.

\subsection{Comparison to top-\% behavior cloning.}

The closest algorithm to \rvs is top-\% BC, which was introduced as a baseline for Decision Transformers \citep{chen2021decision}. This algorithm simply sorts all trajectories in the dataset by return and takes the top $ \rho $ fraction of trajectories to run behavior cloning (for $ \rho \in [0,1]$). The most obvious difference between this algorithm and \rvs is that \rvs allows us to plug in different conditioning functions at test time to produce different policies, while top-\% BC learns only one policy. However, if we want to achieve high returns, the two algorithms are quite similar.

The full statements and proofs of our theoretical results for top-\% BC are deferred to Appendix \ref{app:pbc}. The results are essentially the same as those for \rvs except for two key modifications:

\paragraph{Defining coverage.} The first difference in the analysis is the notion of coverage. For \rvs we needed the return distribution to cover the conditioning function $ f$. For top-\% BC we instead let $ g_\rho$ be the $ 1-\rho$ quantile of the return distribution over trajectories sampled by the behavior $ \beta$ and then define coverage as $ P_\beta(g \geq g_\rho|s) \geq \alpha_\rho$ for all $ s$. This modification is somewhat minor.

\paragraph{Sample size and generalization.} The main difference between \rvs and top-\% BC is that the \rvs algorithm attempts to transfer information gained from low-return trajectories while the top-\% BC algorithm simply throws those trajectories away. This shows up in the formal bounds since for a dataset of size $ N $ the top-\% BC algorithm only uses $ \rho \cdot N$ samples while \rvs uses all $ N$. Depending on the data distribution, competing with the optimal policy may require setting $ \rho$ very close to zero (exponentially small in $ H$) yielding poor sample complexity.

These bounds suggest that \rvs can use generalization across returns to provide improvements in sample complexity over top-\% BC by leveraging all of the data. However, the \rvs model is attempting to learn a richer class of functions that conditions on reward, which may require a larger policy class negating some of this benefit. 
Overall, \rvs should expect to beat top-\% BC if the behavior policy is still providing useful information about how to interact with the environment in low-return trajectories that top-\% BC would throw away.


\section{Experiments}\label{sec:exp}

We have illustrated through theory and some simple examples when we expect \rvs to work, but the theory does not cover all cases that are relevant for practice. In particular, it is not clear how the neural networks trained in practice can leverage generalization across returns. Moreover, one of the key benefits to \rvs approaches (as compared to DP) is that by avoiding the instabilities of non-linear off-policy DP in favor of supervised learning, one might hope that \rvs is more stable in practice. In this section we attempt to test these capabilities first through targeted experiments in a point-mass environment and then by comparisons on standard benchmark data.

Throughout this section we will consider six algorithms, two from each of three categories: 
\begin{enumerate}
    \item {\color{blue}Behavior cloning (BC)}: standard behavior cloning (BC) and percentage behavior cloning (\%BC) that runs BC on the trajectories with the highest returns \citep{chen2021decision}.
    \item {\color{Brown}Dynamic programming (DP)}: TD3+BC \citep{fujimoto2021minimalist} a simple DP-based offline RL approach and IQL \citep{kostrikov2021offlineb} a more stable DP-based offline RL approach.
    \item {\color{OliveGreen}Return-conditioned supervised learning (RCSL)}: RvS \citep{emmons2021rvs} an \rvs approach using simple MLP policies, and DT \citep{chen2021decision} an \rvs approach using transformer policies.
\end{enumerate}
All algorithms are implemented in JAX \citep{jax2018github} using flax \citep{flax2020github} and the jaxrl framework \citep{kostrikovjaxrl}, except for DT which is taken from the original paper. Full details can be found in Appendix \ref{app:exp-details} and code can be found at \url{https://github.com/davidbrandfonbrener/rcsl-paper}.

\subsection{Point-mass datasets}

First, we use targeted experiments to demonstrate how the tabular failure modes illustrated above can arise even in simple deterministic MDPs that may be encountered in continuous control. Specifically, we will focus on the issue of exponential sample complexity discussed in Section \ref{sec:finite}.
We build our datasets in an environment using the Deepmind control suite \citep{tassa2018deepmind} and MuJoCo simulator \citep{mujoco}. The environment consists of a point-mass navigating in a 2-d plane. 

To build an example with exponential sample complexity we construct a navigation task with a goal region in the center of the environment. The dataset is constructed by running a behavior policy that is a random walk that is biased towards the top right of the environment. To construct different levels of reward coverage, we consider the environment and dataset under three different reward functions (ordered by probability of seeing a trajectory with high return, from lowest to highest):
\begin{enumerate}
    \item[(a)] The ``ring of fire'' reward. This reward is 1 within the goal region, -1 in the ring of fire region surrounding the goal, and 0 otherwise
    \item[(b)] The sparse reward. This reward is 1 within the goal region and 0 otherwise. 
    \item[(c)] The dense reward. This reward function is 1 within the goal region and gradually decays with the Euclidean distance outside of it.
\end{enumerate}

\begin{figure}[h]
    \includegraphics[width=0.84\textwidth]{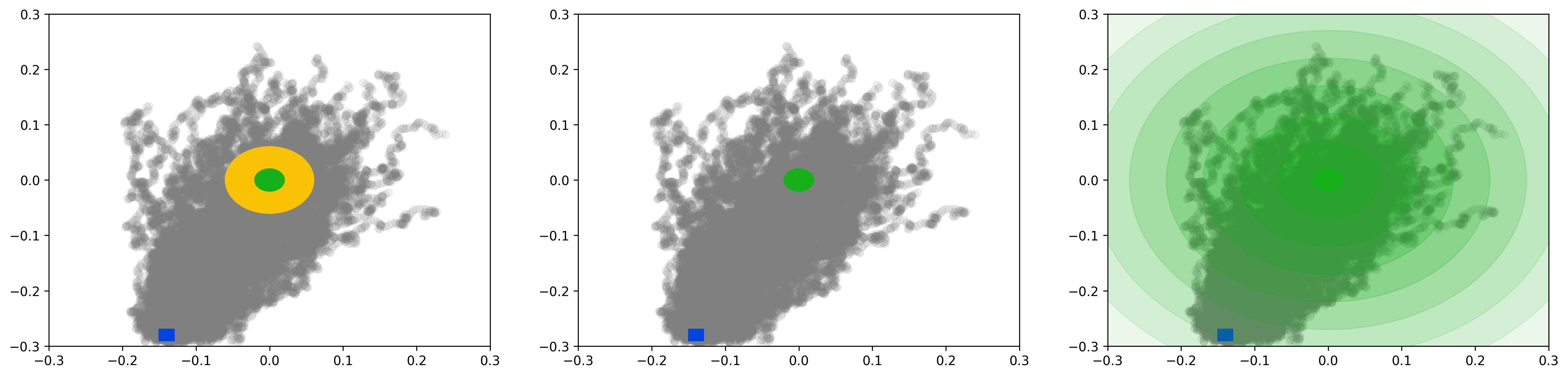}\\
    \includegraphics[width=\textwidth]{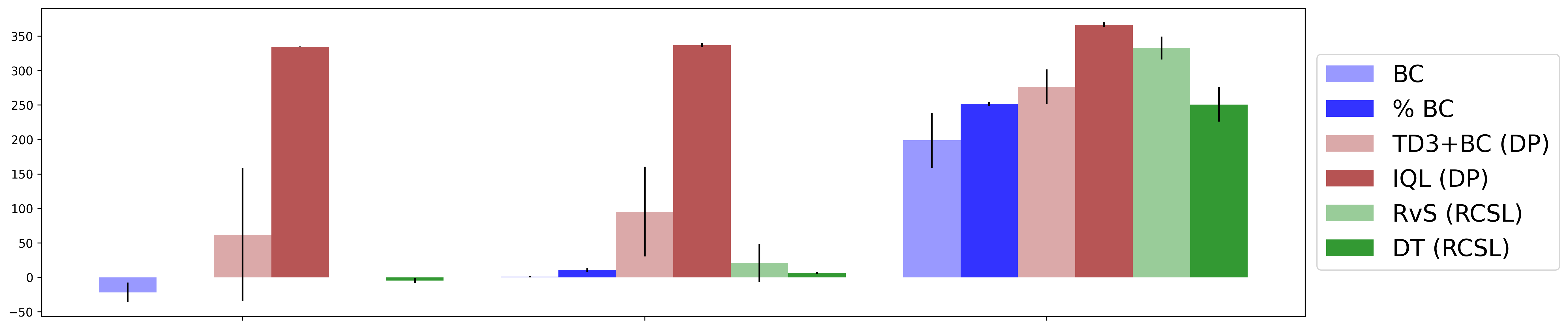}
    
    \quad \quad \quad (a) Ring of fire \quad \quad \quad \quad \quad (b) Sparse \quad \quad \quad \quad \quad \quad \quad  (c) Dense
    
    \caption{\rvs fails under reward functions that lead to exponentially small probability of sampling good trajectories, but can generalize when the reward is dense. Error bars show standard deviation across three seeds. BC methods are in blue, DP methods in brown, and \rvs methods in green.}
    \label{fig:ring}
\end{figure}

Intuitively, the ring of fire reward will cause serious problems for \rvs approaches when combined with the random walk behavior policy. The issue is that any random walk which reached the goal region is highly likely to spend more time in the region of negative rewards than in the actual goal states, since the ring of fire has larger area than the goal. As a result, while they are technically supported by the distribution, it is unlikely to find many trajectories (if any at all) with positive returns in the dataset, let alone near-optimal returns. As a result, the \rvs-based approaches are not even able to learn to achieve positive returns, as seen in Figure \ref{fig:ring}.

The sparse reward is also difficult for the \rvs-based algorithms, for similar reasons; however the problem is less extreme since any trajectory that gets positive reward must go to the goal, so there is signal in the returns indicating where the goal is.
In contrast, the dense reward provides enough signal in the returns that \rvs approaches are able to perform well, although still not as well as IQL. It is also worth noting that because the datset still does not have full coverage of the state-space, simple DP-based algorithms like TD3+BC can struggle with training instability.

\subsection{Benchmark data}

In addition to our targeted experiments we also ran our candidate algorithms on some datasets from the D4RL benchmark \citep{fu2020d4rl}. These are meant to provide more realistic and larger-scale data scenarios. While this also makes these experiments less targeted, we can still see that the insights that we gained in simpler problems can be useful in these larger settings.
We attempt to choose a subset of the datasets with very different properties from eachother. For example, the play data on the ant-maze environment is very diverse and plentiful while the human demonstration data on the pen environment has poor coverage but high values. Results are shown in Figure \ref{fig:d4rl}. And additional results leading to similar conclusions can be found in Appendix \ref{app:exp-extra}.

\begin{wrapfigure}[18]{r}{0.51\textwidth}
    \centering
    \includegraphics[width=0.5\textwidth]{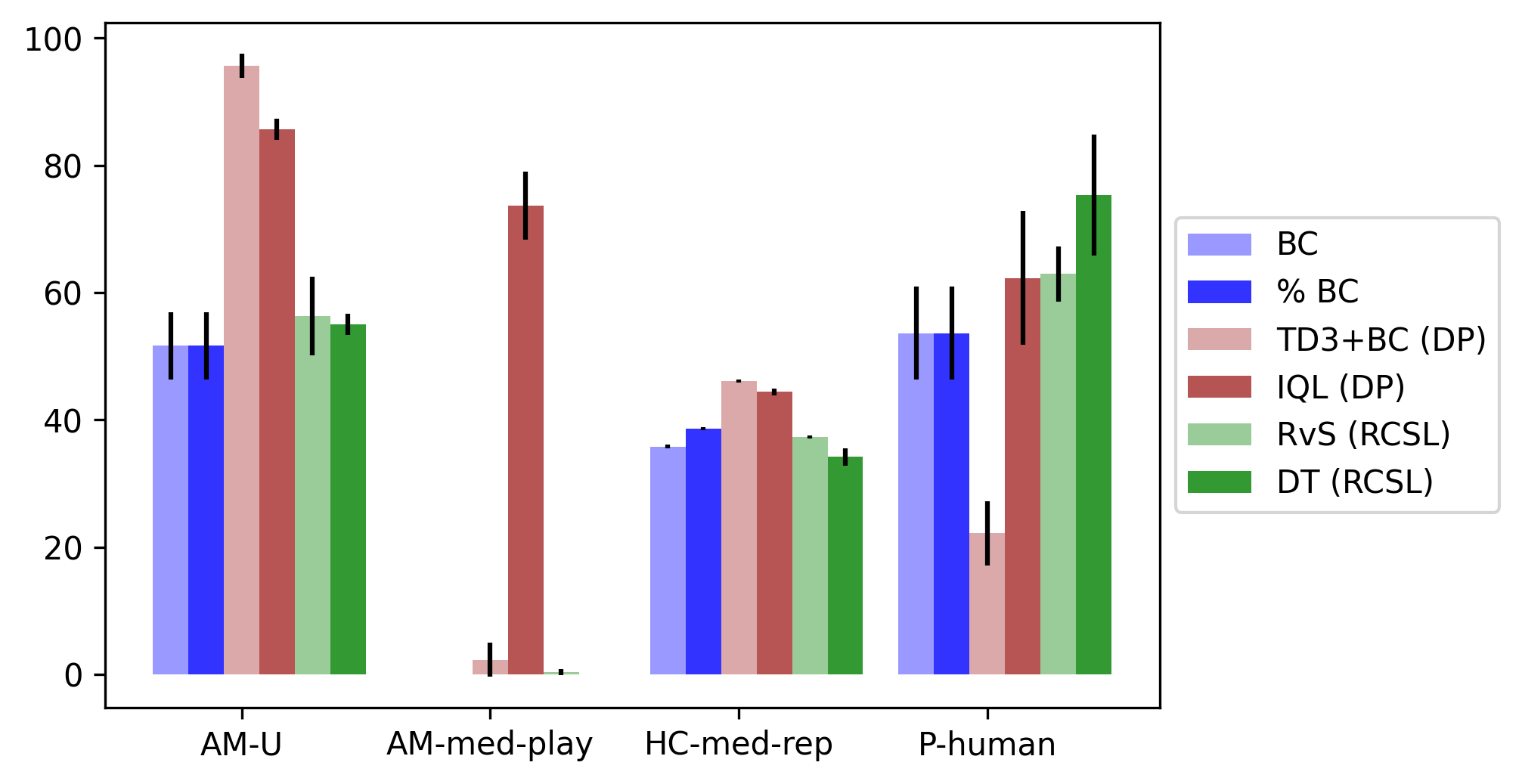}
    \caption{Data from \textsc{antmaze-umaze}, \textsc{antmaze-medium-play}, \textsc{halfcheetah-medium-replay}, and \textsc{pen-human}. Error bars show standard deviation across three seeds. Each algorithm is tuned over 4 values and best performance is reported. }
    \label{fig:d4rl}
\end{wrapfigure}

We find that for most of the datasets DP-based algorithms TD3+BC and IQL outperform both the BC-based algorithms and \rvs-based algorithms. This is especially stark on the \textsc{antmaze} datasets where the behavior policy is highly stochastic, requiring the learner to stitch together trajectories to achieve good performance. While none of these tasks has stochastic dynamics, the issues of return coverage and trajectory stitching persist.

In contrast, \rvs performs well when the behavior policy is already high quality, but not optimal (as in the \textsc{pen-human} task). Since the data is suboptimal and reward is dense, there is opportunity for \rvs to outperform the BC-based methods. Moreover, since the data has poor coverage, standard DP approaches like TD3+BC are highly unstable. IQL is more stable and performs similarly to the \rvs-based algorithms, but is outperformed by DT (perhaps due to the use of history-dependent policies).


\section{Discussion}

Looking back at our results, we can better place \rvs in relation to the more classical BC and DP algorithms. Like BC, \rvs relies on supervised learning and thus inherits its simplicity, elegance, and ease of implementation and debugging. However, it also inherits BC's dependence on the quality of the behavior policy. 
This dependence can be somewhat reduced in (nearly) deterministic environments, where conditioning on high returns can break the bias towards the behavior policy. But, the reliance on trajectory-level information still means that \rvs is fundamentally limited by the quality of the best trajectories in the dataset, which can require a sample complexity exponential in horizon in order to compete with the optimal policy, even in deterministic environments. 



In contrast, DP methods are capable of learning good policies even when the dataset does not contain any high-return trajectories and the environment is stochastic. This represents a fundamental gap between the two approaches that cannot be bridged within the \rvs paradigm.
However, empirically, current deep DP algorithms are not well-behaved. These algorithms are often unstable and difficult to debug, although recent work has started to alleviate these issues somewhat~\cite{kostrikov2021offlineb}. 

In sum, for tasks where the requirements for \rvs to perform well are met, it is an excellent practical choice, with great advantages in simplicity over DP. Since many real-world tasks of relevance have these attributes, \rvs techniques could have substantial impact. But as a general learning paradigm, \rvs is fundamentally limited in ways that DP is not.

\subsection*{Acknowledgments}
This work was partially supported by NSF RI-1816753, NSF CAREER CIF 1845360, NSF CHS-1901091, NSF Scale MoDL DMS 2134216, Capital One and Samsung Electronics.
DB was supported by the Department of Defense (DoD) through the National Defense Science \& Engineering Graduate Fellowship (NDSEG) Program.

\bibliography{rl.bib}

\newpage

\section*{Checklist}


\begin{enumerate}

\item For all authors...
\begin{enumerate}
  \item Do the main claims made in the abstract and introduction accurately reflect the paper's contributions and scope?
    \answerYes{}
  \item Did you describe the limitations of your work?
    \answerYes{}
  \item Did you discuss any potential negative societal impacts of your work?
    \answerYes{See Appendix \ref{app:societal}}
  \item Have you read the ethics review guidelines and ensured that your paper conforms to them?
    \answerYes{}
\end{enumerate}

\item If you are including theoretical results...
\begin{enumerate}
  \item Did you state the full set of assumptions of all theoretical results?
    \answerYes{}
        \item Did you include complete proofs of all theoretical results?
    \answerYes{}
\end{enumerate}

\item If you ran experiments...
\begin{enumerate}
  \item Did you include the code, data, and instructions needed to reproduce the main experimental results (either in the supplemental material or as a URL)?
    \answerYes{}
  \item Did you specify all the training details (e.g., data splits, hyperparameters, how they were chosen)?
    \answerYes{} See Appendix \ref{app:exp-details}
        \item Did you report error bars (e.g., with respect to the random seed after running experiments multiple times)?
    \answerYes{}
        \item Did you include the total amount of compute and the type of resources used (e.g., type of GPUs, internal cluster, or cloud provider)?
    \answerYes{} See Appendix \ref{app:exp-details}
\end{enumerate}

\item If you are using existing assets (e.g., code, data, models) or curating/releasing new assets...
\begin{enumerate}
  \item If your work uses existing assets, did you cite the creators?
    \answerYes{}
  \item Did you mention the license of the assets?
    \answerYes{} See Appendix \ref{app:exp-details}
  \item Did you include any new assets either in the supplemental material or as a URL?
    \answerNA{}
  \item Did you discuss whether and how consent was obtained from people whose data you're using/curating?
    \answerNA{}
  \item Did you discuss whether the data you are using/curating contains personally identifiable information or offensive content?
    \answerNA{}
\end{enumerate}

\item If you used crowdsourcing or conducted research with human subjects...
\begin{enumerate}
  \item Did you include the full text of instructions given to participants and screenshots, if applicable?
    \answerNA{}
  \item Did you describe any potential participant risks, with links to Institutional Review Board (IRB) approvals, if applicable?
    \answerNA{}
  \item Did you include the estimated hourly wage paid to participants and the total amount spent on participant compensation?
    \answerNA{}
\end{enumerate}

\end{enumerate}

\newpage

\appendix

\section{Extended experimental results}\label{app:exp-extra}

Here we present extended versions of the D4RL experiments. We use the same setup as in Section \ref{sec:exp}, but run each of the algorithms on three different datasets in each environment. Explicitly we show results on \textsc{antmaze-umaze}, \textsc{antmaze-medium-play}, and \textsc{antmaze-large-play} in Figure \ref{fig:antmaze}. Then we show results on \textsc{halfcheetah-medium}, \textsc{halfcheetah-medium-replay}, and \textsc{halfcheetah-medium-expert} in Figure \ref{fig:halfcheetah}. Finally we show results on \textsc{pen-human}, \textsc{pen-cloned}, and \textsc{pen-expert} in Figure \ref{fig:pen}. 

These experiments corroborate the story from the main text. Without return coverage (as in the larger antmaze tasks), RCSL can fail dramatically. But in the case with return coverage but poor state coverage (as in the pen human dataset that only has 25 trajectories), RCSL can beat DP. However we see that with larger datasets that yield more coverage, DP recovers it's performance (as in pen expert which has 5000 trajectories, or 200x the amount of data as in the human dataset).

\begin{figure}[h]
    \centering
    \includegraphics[width=0.5\textwidth]{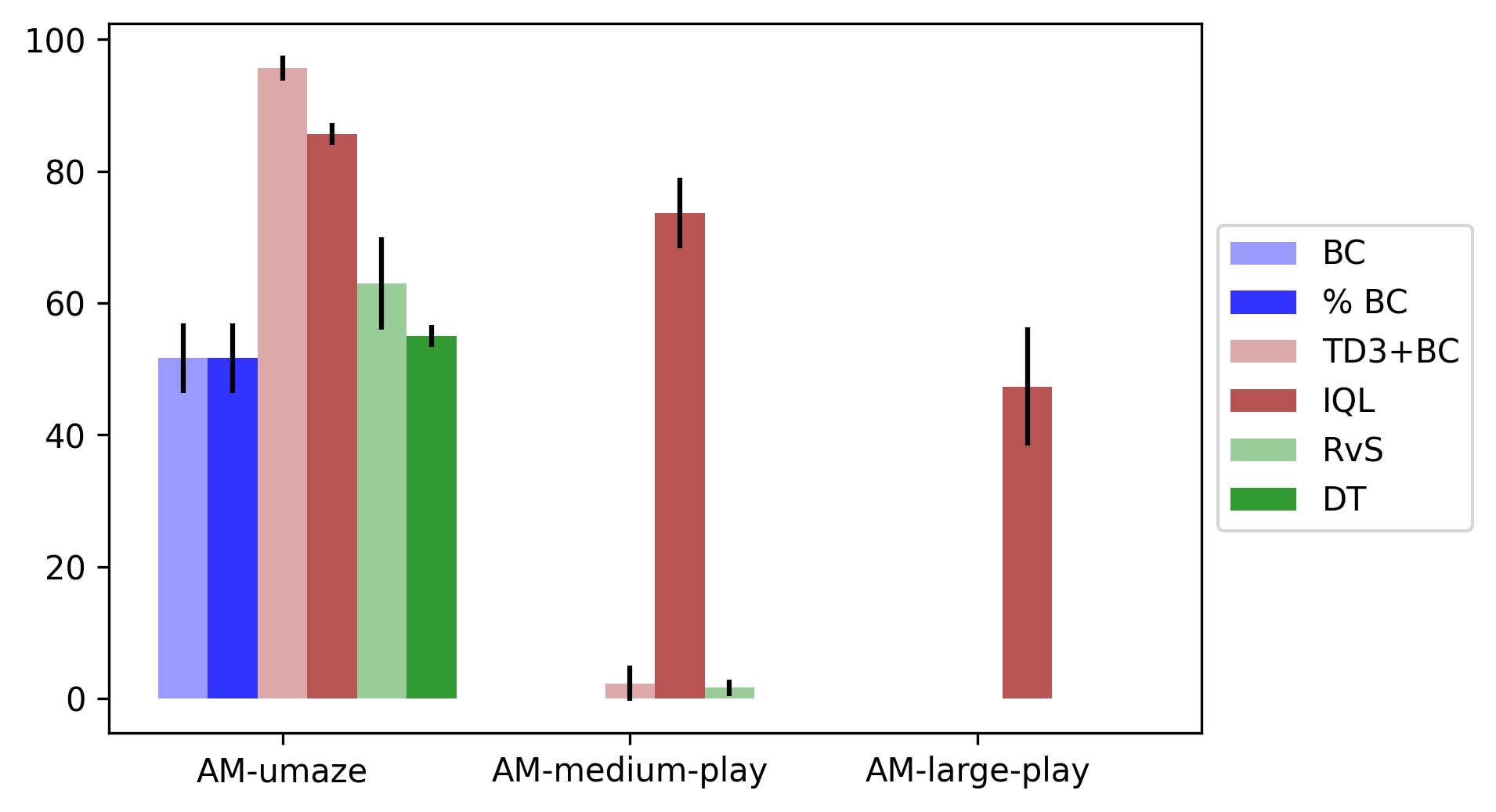}
    \caption{Experimental results on antmaze datasets.}
    \label{fig:antmaze}
\end{figure}

\begin{figure}[h]
    \centering
    \includegraphics[width=0.5\textwidth]{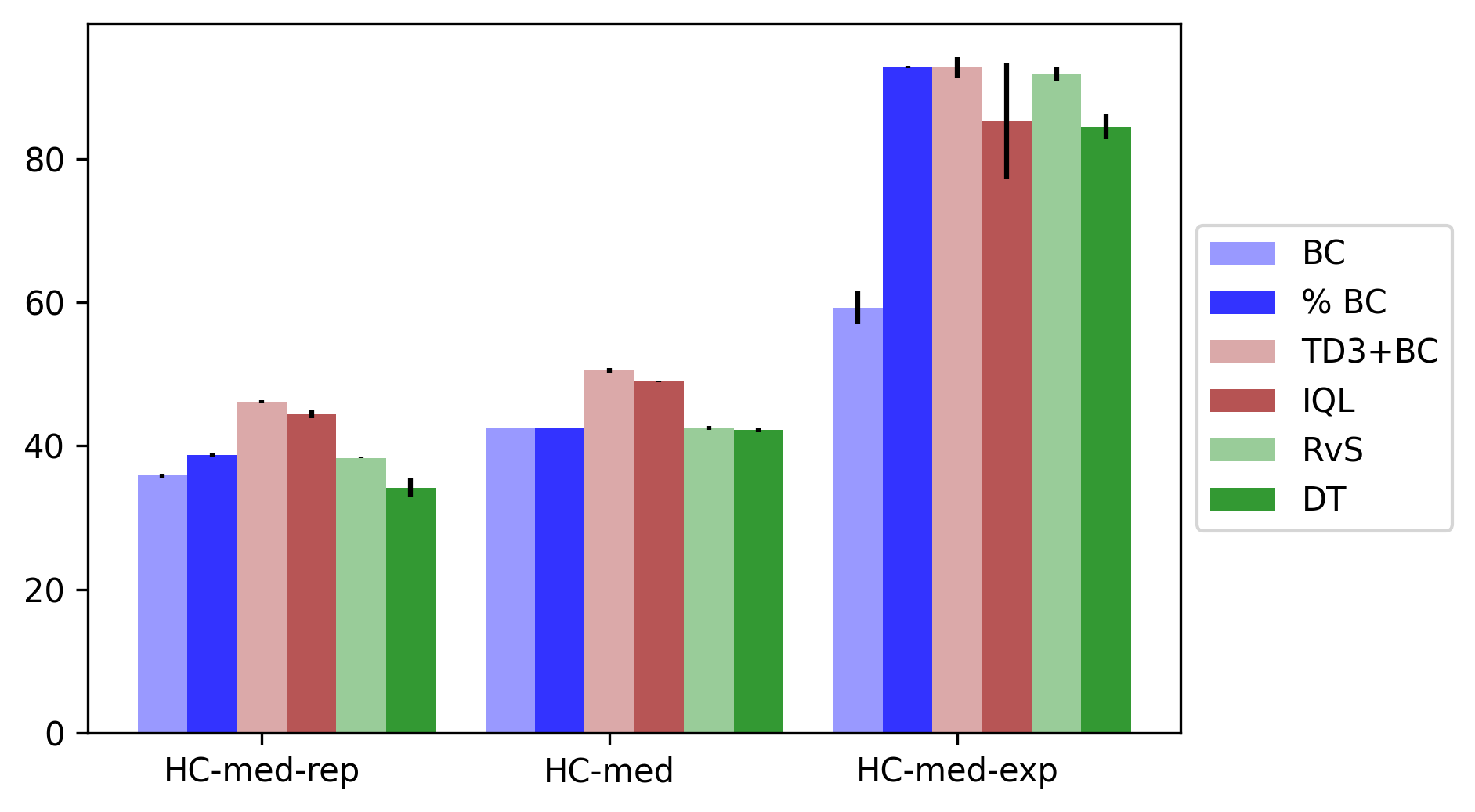}
    \caption{Experimental results on halfcheetah datasets.}
    \label{fig:halfcheetah}
\end{figure}

\begin{figure}[h]
    \centering
    \includegraphics[width=0.5\textwidth]{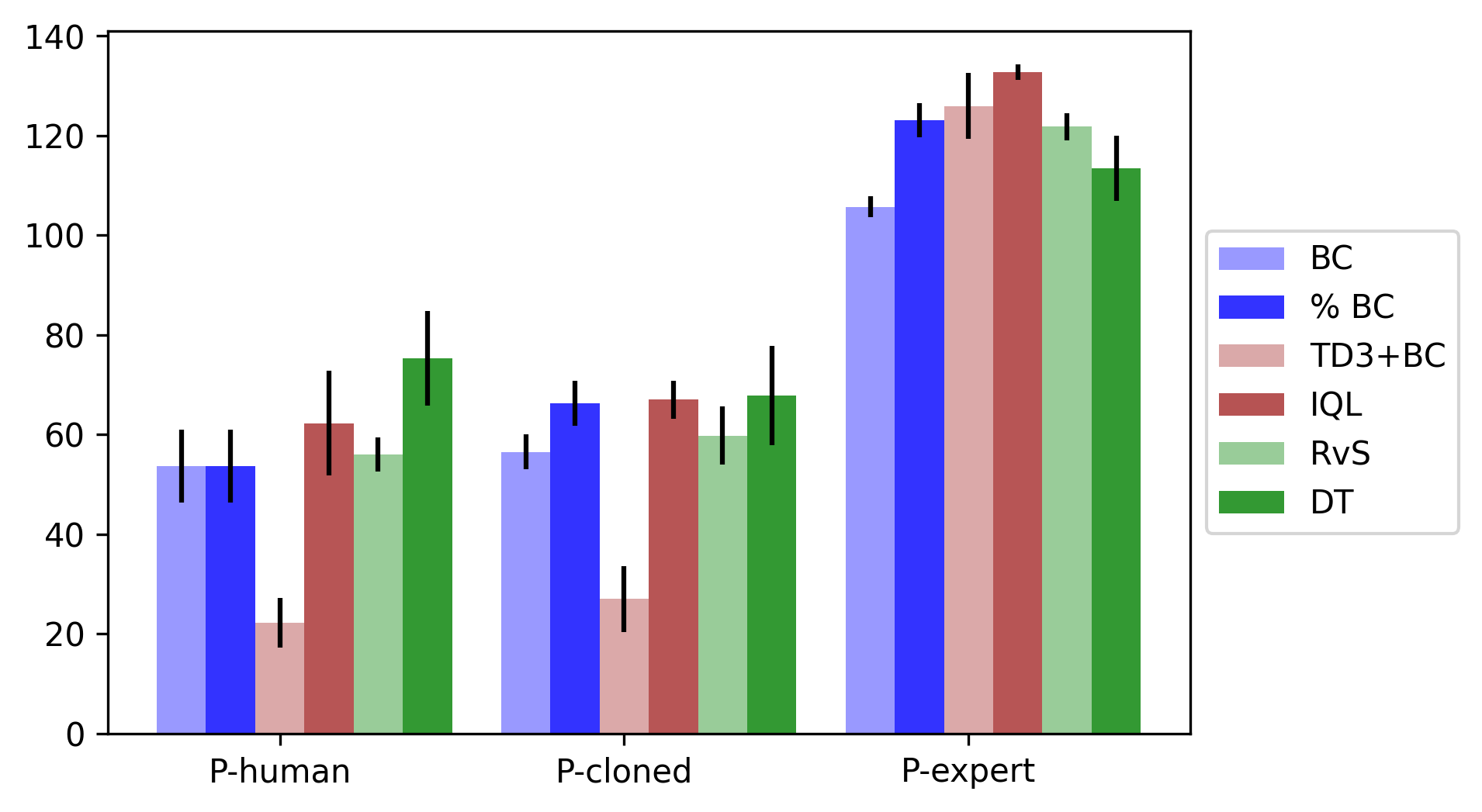}
    \caption{Experimental results on pen datasets.}
    \label{fig:pen}
\end{figure}




\section{Trajectory stitching}\label{app:stitch}

\subsection{Theory}

\begin{figure}[h]
       \centering
       \includegraphics[width=0.3\textwidth]{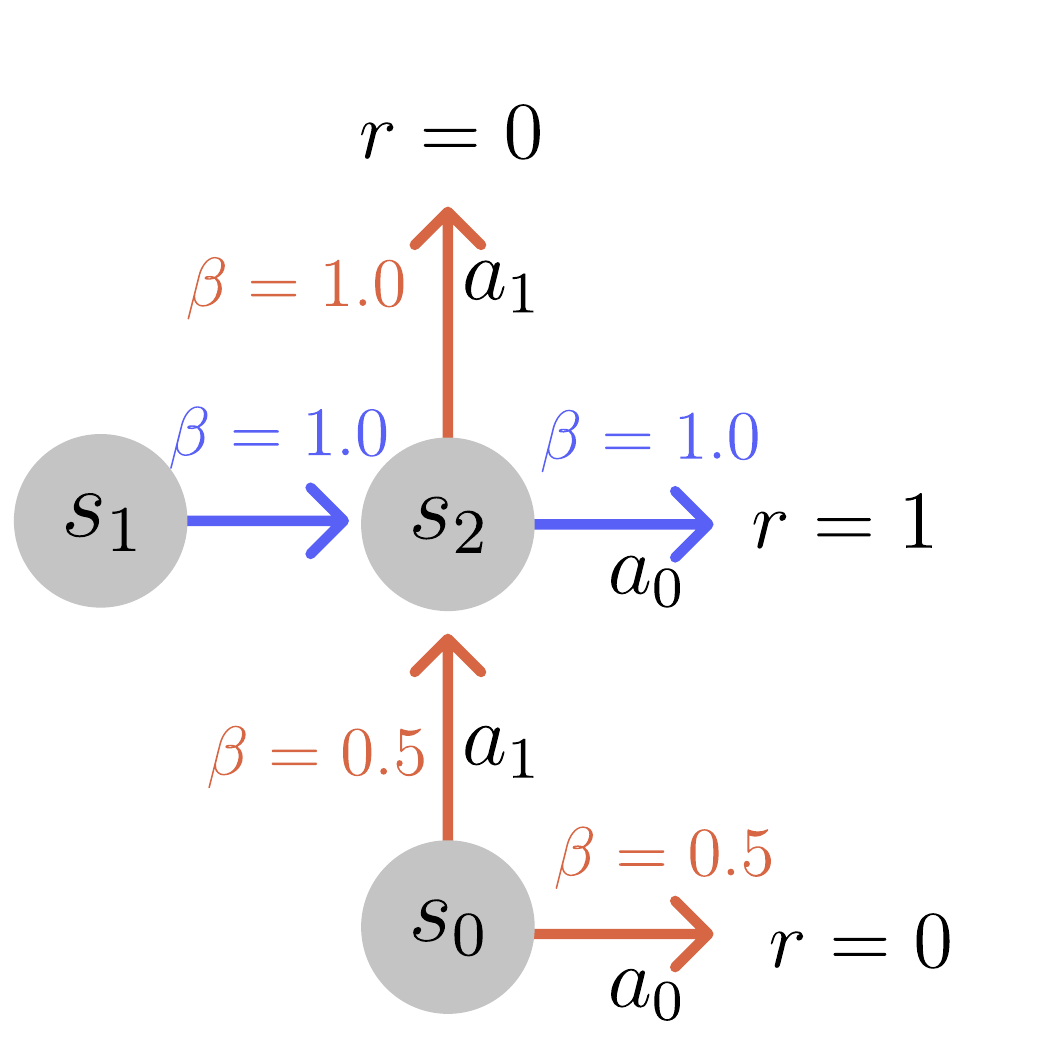}
        \caption{An example where \rvs fails to stitch trajectories.}
        \label{fig:stitch}
\end{figure}
 
A common goal from the offline RL literature is to be able to stitch together previously collected trajectories to outperform the behavior policy. This is in general not possible with \rvs. 
The main issue here is that \rvs is using \emph{trajectory level} information during training, which precludes combining information across trajectories.
In this example we show that even with infinite data, when attempting to combine two datastreams using standard approaches to conditioning \rvs can fail to recover the optimal policy.

Consider the MDP illustrated in Figure \ref{fig:stitch} with three states $ s_0, s_1, s_2$ and horizon $ H = 2$. All transitions and rewards are deterministic as shown. We consider the case where data has been collected by two different processes. One process (illustrated in red) consists of episodes that always start at $ s_0$ and chooses the first action uniformly but chooses the bad action $ a_0$ deterministically from $ s_2$. The other process (illustrated in blue) consists of trajectories that always start at $ s_1 $ and deterministically go towards the good action, receiving reward of 1. We will consider what happens to \rvs at test time when initialized at $ s_0$.

The data does not contain any trajectories that begin in $ s_0$ and select $ a_1$ to transition to $ s_2 $ followed by $ a_1$, which is the optimal decision. But, the data does have enough information to stitch together the optimal trajectory from $ s_0$, and it is clear to see that DP-based approaches would easily find the optimal policy.

For \rvs, if we condition on optimal return $ g  = 1$, we get that $ \pi(\cdot |s_1, g=1) $ is undefined since we only observe trajectories with $ g = 0$ that originate at $ s_0$. To get a well-defined policy, we must set $ f(s_0) = 0$, but then $ \pi(a_1|s_1, g=0) = 0.5$. 
Thus, $ \pi$ will never choose the optimal path with probability larger than 0.5, for any conditioning function $ f$. Moreover, the conditioning function that does lead to success is non-standard: $ f(s_0) = 0, f(s_2) = 1$. For the standard approach to conditioning of setting the initial value and decreasing over time with observed rewards, \rvs will never achieve non-zero reward from $ s_0$. 

Note that DT-style learning where we condition on the entire history of states rather than just the current state can perform even worse since $ P_{data}(a_1|s_0, a_0, s_2, g=1) = 0$, i.e. even using the non-standard conditioning function described above will not fix things. 
Also, it is worth mentioning that it is possible that conditioning on the out-of-distribution return $ g=1$ from $ s_0$ could work due to extrapolation of a neural network policy. However, as we will see in the experiments section, this does not happen empirically in controlled settings.  

\subsection{Experiments}

The above example does not take into account the possibility of generalization out of distribution (i.e. when conditioning on returns that were not observed in the data). To test whether generalization could lead to stitching we construct two datasets: stitch-easy and stitch-hard. Both datasets use the same simple point-mass environment with sparse rewards as before, but now we introduce a wall into the environment to limit the paths to the goal. The stitch-easy dataset contains two types of trajectories: some beginning from the initial state region and moving upwards (with added Gaussian noise in the actions) and some beginning from the left side of the environment and moving towards the goal (with added Gaussian noise in the actions). This is ``easy'' since just following the behavior policy for the first half of the trajectory leads to states where the dataset indicates how to reach the goal. We also create the stitch-hard dataset which includes a third type of trajectory that begins from the initial state and goes to the right (mirroring the tabular example). This is ``hard'' since the dominant action from the behavior in the initial state is now to go right rather than to move towards the goal-reaching trajectories. This acts as a distraction for methods that are biased towards the behavior. Datasets and results are illustrated in Figure \ref{fig:stitch-experiment}.

\begin{figure}[h]
    \centering
    \begin{tabular}{cc}
        \includegraphics[width=0.3\textwidth]{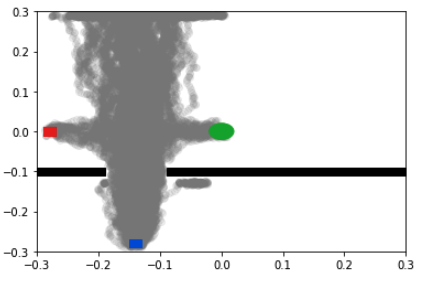} &  \includegraphics[width=0.3\textwidth]{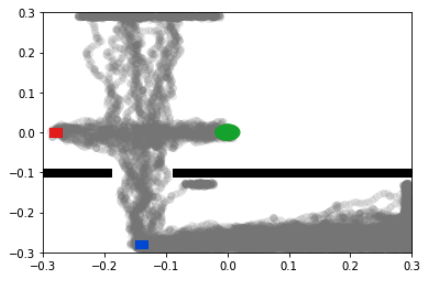} \\
          \includegraphics[width=0.3\textwidth]{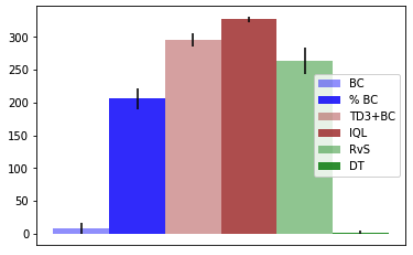} & \includegraphics[width=0.3\textwidth]{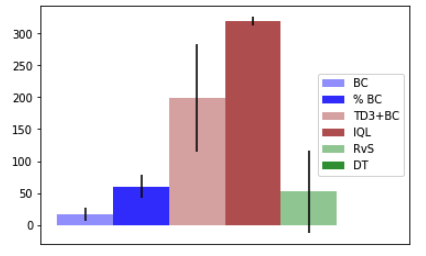}\\
         (a) Stitch-easy & (b) Stitch-hard
    \end{tabular}
    
    \caption{Results on two datasets that require stitching.}
    \label{fig:stitch-experiment}
\end{figure}

We see that on the stitch-easy dataset RvS is able to perform almost as well as the DP algorithms and better than \%BC. This indicates that it is able to follow the behavior when conditioning on an out-of-distribution return until it reaches a state where the return directs it to the goal. In contrast, DT totally fails on this task since it conditions on the entire history of the trajectory. Since the dataset only contains trajectories from the initial state that continue to the top of the board, DT always reproduces such trajectories from the initial state and does not stitch together trajectories. 

In the stitch-hard dataset, we see that RvS fails as well and does not outperform \%BC. This indicates that indeed, RvS can be distracted by the distractor trajectories from the initial state. The conditioning itself was not what cause the stitching in the stitch-easy dataset, but rather the learned policy simply defaults to the behavior. This can be beneficial in some problems, but prevents trajectory stitching that might allow the learned policy to dramatically outperform the behavior. 
TD3+BC also struggles here, likely due to some combination of instability and the BC regularization causing issues due to the distractor actions.

\section{Proofs}\label{app:proof}

\subsection{Proof of Theorem \ref{thm:infinite}}\label{app:infinite}

\begin{proof}
Let $ g(s_1, a_{1:H}) $ be the value of the return by rolling out the open loop sequence of actions $ a_{1:H}$ under the deterministic dynamics induced by $ T $ and $ r$. Then we can write
\begin{align}
    \E_{s_1}[f(s_1)] - J(\pi_f)  &= \E_{s_1}\left[\E_{\pi_f|s_1} [f(s_1)  - g_1] \right] \\
    &= \E_{s_1} \left[\E_{a_{1:H} \sim \pi_f|s_1} [f(s_1) - g(s_1, a_{1:H}) ] \right] \\&\qquad+ \E_{s_1}\left[\E_{a_{1:H} \sim \pi_f|s_1} [ g(s_1, a_{1:H}) - g_1 ] \right]\\
    &\leq \E_{s_1} \left[\E_{a_{1:H} \sim \pi_f|s_1} [f(s_1) - g(s_1, a_{1:H})] \right] + \epsilon H^2. \label{eq:first}
\end{align}
where the last step follows by bounding the magnitude of the difference between $ g_1 $ and $ g(s_1, a_{1:H})$ by $H $ and applying a union bound over the $ H $ steps in the trajectory (using the near determinism assumption), namely:
\begin{align}
        H \cdot \sup_{s_1} \bigcup_{t} P_{a_t \sim \pi_f|s_1}(r_t \neq r(s_t, a_t) \text{ or } s_{t+1} \neq T(s_t, a_t)) \leq \epsilon H^2.
\end{align}
Now we consider the first term from eq. (\ref{eq:first}). Again bounding the magnitude of the difference by $H$ we get that
\begin{align}
    \E_{s_1} \left[\E_{a_{1:H} \sim \pi_f|s_1} [f(s_1) - g(s_1, a_{1:H})] \right] &\leq \E_{s_1} \int_{a_{1:H}} P_{\pi_f}(a_{1:H}|s_1) \1[g(s_1, a_{1:H}) \neq f(s_1)] H\label{eq:second}
\end{align}
To bound this term, we will more carefully consider what happens under the distribution $ P_{\pi_f}$. 

To simplify notation, let $ \bar s_t = T(s_1, a_{1:t-1})$ be the result of following the deterministic dynamics defined by $ T $ up until step $ t$.
Expanding it out, applying the near determinism, the consistency of $ f $, the coverage assumption, canceling some terms, and then inducting we see that:
\begin{align}
    &P_{\pi_f}(a_{1:H}|s_1) = \pi_f(a_1|s_1) \int_{s_2} P(s_2|s_1, a_1) P_{\pi_f}(a_{2:H}|s_1, s_2)\\
    &\leq \pi_f(a_1|s_1) P_{\pi_f}(a_{2:H}|s_1, \bar s_2) + \epsilon\\
    &=  \beta(a_1|s_1)\frac{P_\beta(g_1 = f(s_1)|s_1, a_1)}{P_\beta(g_1 = f(s_1) | s_1)} P_{\pi_f}(a_{2:H}|s_1, \bar s_2) + \epsilon\\
    &\leq \beta(a_1|s_1)\frac{\epsilon + P_\beta(g_1 - r(s_1, a_1) = f(s_1) - r(s_1, a_1)|s_1, a_1, \bar s_2)}{P_\beta(g_1 = f(s_1) | s_1)} P_{\pi_f}(a_{2:H}|s_1, \bar s_2) + \epsilon\\
    &= \beta(a_1|s_1)\frac{\epsilon + P_\beta(g_2 = f(\bar s_2)| \bar s_2)}{P_\beta(g_1 = f(s_1) | s_1)} P_{\pi_f}(a_{2:H}|s_1, \bar s_2) + \epsilon\\
    &\leq \beta(a_1|s_1)\frac{P_\beta(g_2 = f(\bar s_2)| \bar s_2)}{P_\beta(g_1 = f(s_1) | s_1)} P_{\pi_f}(a_{2:H}|s_1, \bar s_2) + \epsilon\left(\frac{1}{\alpha_f} + 1\right)\\
    &\leq \beta(a_1|s_1)\beta(a_2| \bar s_2) \frac{\cancel{ P_\beta(g_2 = f(\bar s_2)| \bar s_2)}}{P_\beta(g_1 = f(s_1) | s_1)}\cdot \frac{P_\beta(g_2 = f(\bar s_2)| \bar s_2, a_2)}{\cancel{P_\beta(g_2 = f(\bar s_2)| \bar s_2)}} P_{\pi_f}(a_{3:H}|s_1, \bar s_3)) \\
    &\qquad+ 2 \epsilon\left(\frac{1}{\alpha_f} + 1\right)\\
    &\leq \prod_{t=1}^H \beta(a_t| \bar s_t)  \frac{P_\beta(g_H = f(\bar s_H) | \bar s_H, a_h)}{P_\beta(g_1 = f(s_1) | s_1)} + H\epsilon\left(\frac{1}{\alpha_f} + 1\right)\\
    &=  \prod_{t=1}^H \beta(a_t| \bar s_t) \frac{\1[g(s_1, a_{1:H}) =  f(s_1)]}{P_\beta(g_1 = f(s_1) | s_1)} + H\epsilon\left(\frac{1}{\alpha_f} + 1\right)
\end{align}
where the last step follows from the determinism of the trajectory that determines $ \bar s_H$ and the consistency of $ f$.
Plugging this back into eq (\ref{eq:second}) and noticing that the two indicator functions can never both be 1, we get that:
\begin{align}
     \E_{s_1} \left[\E_{a_{1:H} \sim \pi_f|s_1} [ f(s_1) - g(s_1, a_{1:H})] \right] \leq H^2\epsilon\left(\frac{1}{\alpha_f} + 1\right)
\end{align}
Plugging this back into eq (\ref{eq:first}) yields the result.
\end{proof}


\subsection{Proof of Corollary \ref{cor:soft-infinite}}\label{app:soft-infinite}

\begin{proof}
We need to define a function $ f $ so that $ \E[f(s_1)] $ is approximately $ J(\pi^*)$. To do this, note that there exists a deterministic optimal policy $ \pi^*$, and since the environment dynamics are nearly deterministic we can set $ f(s_1) $ to be the return of $ \pi^* $ under the deterministic dynamics. To do this, let $ T^{\pi^*}(s_1, t)$ represent the state reached by running $ \pi^*$ from $ s_1$ for $ t $ steps under the deterministic dynamics defined by $ T$. Then:
\begin{align}
    f(s_1) = \sum_{t=1}^H r(T^{\pi^*}( s_1, t), \pi^*(T^{\pi^*}(s_1, t)) )
\end{align}
Now we have as in the proof of Theorem \ref{thm:infinite} that the probability that $ g \neq f(s)$ is bounded by $ \epsilon H$, so that
\begin{align}
    \E_{s_1}[f(s_1)] - J(\pi^*) = \E_{s_1}[\E_{g \sim \pi^*|s_1}[f(s_1) - g]] \leq \E_{s_1}[P_{\pi^*}(g\neq f(s_1)|s_1)\cdot H] \leq \epsilon H^2
\end{align}
Combining this with Theorem \ref{thm:infinite} yields the result.
\end{proof}

\subsection{Proof of Theorem \ref{thm:finite}}\label{app:finite}

First we prove the following Lemma. This can be seen as a finite-horizon analog to results from \citet{achiam2017constrained}.

\begin{lemma}\label{lem:l1}
Let $ d_\pi $ refer to the marginal distribution of $ P_\pi$ over states only. For any two policies $ \pi, \pi'$ we have:
\begin{align}
    \|d_\pi - d_{\pi'} \|_1 \leq 2 H \cdot \E_{s \sim d_\pi}[TV(\pi(\cdot|s)\| \hat \pi'(\cdot |s))]
\end{align}
\end{lemma}

\begin{proof}
    First we will define a few useful objects. Let $ d_\pi^h(s) = P_\pi(s_h = s)$. Let $ \Delta_h = \|d_\pi^h(s) - d_{\pi'}^h(s)\|_1$. Let $ \delta_h = 2 \E_{s\sim d_\pi^h}[TV(\pi(\cdot|s)\| \hat \pi'(\cdot |s))]$.
    
    Now we claim that $ \Delta_h \leq \delta_{h-1} + \Delta_{h-1}$ for $ h > 1$ and $ \Delta_1 = 0$.
    
    To see this, consider some fixed $ h$. Note that $ d_\pi^h(s) = \int_{s'}d_\pi^{h-1}(s')\int_{a'} \pi(a'|s')P(s|s', a') $. Then expanding the definitions and adding and subtracting we see that 
    \begin{align}
        \Delta_h &= \int_s |d_\pi^h(s) - d_{\pi'}^h(s)| \\
        &\leq \int_s \bigg|\int_{s'}d_\pi^{h-1}(s')\int_{a'} (\pi(a'|s') - \pi'(a'|s')) P(s|s', a') \bigg| \\&\qquad + \int_s \bigg|\int_{s'}(d_\pi^{h-1}(s') - d_{\pi'}^{h-1}(s')) \int_{a'} \pi'(a'|s') P(s|s', a') \bigg|\\
        &\leq 2 \E_{s\sim d_\pi^{h-1}}[TV(\pi(\cdot|s)\| \hat \pi'(\cdot |s))] + \|d_\pi^{h-1} - d_{\pi'}^{h-1}\|_1 = \delta_{h-1} + \Delta_{h-1}.
    \end{align}
    
    Now applying the claim and the definition of $ d_\pi$ we get that 
    \begin{align}
        \|d_\pi - d_{\pi'}\|_1 \leq \frac{1}{H}\sum_{h=1}^H \Delta_h \leq \frac{1}{H}\sum_{h=1}^H \sum_{j=1}^{h-1} \delta_j \leq H\frac{1}{H}\sum_{h=1}^H \delta_h = 2 H \cdot \E_{s \sim d_\pi}[TV(\pi(\cdot|s)\| \hat \pi'(\cdot |s))].
    \end{align}
\end{proof}

Now we can prove the Theorem.
\begin{proof}
Applying the definition of $J$ and Lemma \ref{lem:l1}, we get
\begin{align}
     J(\pi_f) - J(\hat \pi_f) &= H (\E_{P_{\pi_f}}[r(s,a)] - \E_{P_{\hat \pi_f}}[r(s,a)] )\\
     &\leq H \| d_{\pi_f} - d_{\hat \pi_f} \|_1\\
     &\leq 2 \cdot \E_{s \sim d_{\pi_f}}[TV(\pi_f(\cdot|s) \| \hat \pi_f(\cdot|s))] H^2
\end{align}
Expanding definitions, using the multiply and divide trick, and applying the assumptions:
\begin{align}
    2 \cdot \E_{s \sim d_{\pi_f}}[TV(\pi_f(\cdot|s) \| \hat \pi_f(\cdot|s))] &=  \E_{s \sim d_{\pi_f}} \left[ \int_a |P_\beta(a|s, f(s)) - \hat \pi(a|s, f(s))|   \right]\\
    &= \E_{s \sim d_{\pi_f}} \left[ \frac{P_\beta(f(s)|s)}{P_\beta(f(s)|s)} \int_a  |P_\beta(a|s, f(s)) - \hat \pi(a|s, f(s))| \right]\\
    &\leq  \frac{C_f}{\alpha_f} \E_{s \sim d_{\beta}}\left[ P_\beta(f(s)|s) \int_a  |P_\beta(a|s, f(s)) - \hat \pi(a|s, f(s))|\right]\\
    &\leq \frac{C_f}{\alpha_f} \E_{s \sim d_{\beta}}\left[\int_g P_\beta(g|s) \int_a  |P_\beta(a|s, g) - \hat \pi(a|s, g)|\right]\\
    &= 2 \frac{C_f}{\alpha_f} \E_{s \sim d_{\pi_f}, g \sim P_\beta|s}[TV(P_\beta(\cdot|s, g) \| \hat \pi(\cdot|s, g))]\\
    &\leq \frac{C_f}{\alpha_f} \sqrt{2 L(\hat \pi)}
\end{align}
where the last step comes from Pinsker's inequality.
Combining with the above bound on the difference in expected values yields the result. 
\end{proof}

\subsection{Proof of Corollary~\ref{cor:finite}}
\label{cor:finite_proof}
\begin{proof}
We may write~$L(\pi) = \bar L(\pi) - H_\beta$, where~$H_\beta = -\E_{(s,a,g) \sim P_\beta}[\log P_\beta(a|s,g)]$ and
\[
\bar L(\pi) := - \E_{(s, a, g) \sim P_\beta}[\log \pi(a|s,g)]
\]
is the cross-entropy loss.
Denoting~$\pi^\dagger \in \arg\min_{\pi \in \Pi} L(\pi)$, we have
\begin{align*}
    L(\hat \pi) &= L(\hat \pi) - L(\pi^\dagger) + L(\pi^\dagger) \leq \bar L(\hat \pi) - \bar L(\pi^\dagger) + \epsilon_{approx}.
\end{align*}
Denoting~$\hat L$ the empirical cross-entropy loss that is minimized by~$\hat \pi$, we may further decompose
\begin{align*}
    \bar L(\hat \pi) - \bar L(\pi^\dagger) &= \bar L(\hat \pi) - \hat L(\hat \pi) + \hat L(\hat \pi) - \hat L(\pi^\dagger) + \hat L(\pi^\dagger) - \bar L(\pi^\dagger) \\
    &\leq 2 \sup_{\pi \in \Pi} |\bar L(\pi) - \hat L(\pi)|
\end{align*}
Under the assumptions on bounded loss differences, we may bound this, e.g., using McDiarmid's inequality and a union bound on~$\Pi$ to obtain the final result.
\end{proof}

\subsection{Top-\% BC}\label{app:pbc}

\begin{theorem}[Alignment with respect to quantile]
 Let $ g_\rho $ be the $ 1-\rho$ quantile of the return distribution induced by $ \beta$ over all initial states. Let $ \pi_\rho = P_\beta(a|s, g \geq g_\rho)$. Assume the following:
\begin{enumerate}
    \item Coverage: $P_\beta(s_1|g\geq g_\rho) \geq \alpha_\rho $ for all initial states $ s_1$.
    \item Near determinism: $ P(r \neq r(s, a) \text{ or } s' \neq T(s, a) | s,a ) \leq \epsilon$ at all $ s, a $ for some functions $ T$ and $ r $. Note that this does not constrain the stochasticity of the initial state at all.
\end{enumerate}
Then
\begin{align}
    g_\rho - J(\pi_\rho) \leq \epsilon\left( \frac{1}{\alpha_\rho} + 2\right)H^2.
\end{align}
\end{theorem}

\begin{proof}
The proof essentially follows the same argument as Theorem \ref{thm:infinite} with $ f(s_1)$ replaced by $ g_\rho$. The main difference comes from the fact that 
\begin{align}
    \pi_\rho(a|s) = P_\beta(a|s, g \geq g_\rho) = \beta(a|s) \frac{P_\beta(g\geq g_\rho | s,a)}{P_\beta(g\geq g_\rho|s)}
\end{align}
Explicitly, we have similar to before that:
\begin{align}
    g_\rho - J(\pi_\rho) &= \E_{s_1}[\E_{\pi_\rho|s_1}[g_\rho - g_1]]\\
    &\leq \E_{s_1}\E_{a_{1:H}\sim \pi_f|s_1}[g_\rho - g(s_1, a_{1:H})] + \epsilon H^2.\\
    &\leq \E_{s_1}\E_{a_{1:H}\sim \pi_f|s_1}[\1[g(s_1, a_{1:H}) < g_\rho]] \cdot H + \epsilon H^2.\label{eq:pbc-proof}
\end{align}
We now define $ \bar s_t = T(s_1, a_{1:t-1})$ to be the state at step $ t$ under the determinisitic dynamics and similarly $ \bar r_t = r(\bar s_t, a_t)$ the reward under deterministic dynamics. 
Then again mirroring the proof above, we have that 
\begin{align}
    &P_{\pi_\rho}(a_{1:H}|s_1) \leq \pi_\rho(a_1|s_1) P_{\pi_\rho}(a_{2:H}|s_1, \bar s_2, \bar r_1) + \epsilon\\
    &= \beta(a_1|s_1) \frac{P_\beta(g_1 \geq g_\rho|s_1, a_1)}{P_\beta(g_1 \geq g_\rho|s_1)}P_{\pi_\rho}(a_{2:H}|s_1, \bar s_2, \bar r_1) + \epsilon\\
    &\leq \beta(a_1|s_1) \frac{\epsilon + P_\beta(g_1 \geq g_\rho|\bar s_2, \bar r_1, a_1)}{P_\beta(g_1 \geq g_\rho|s_1)}P_{\pi_\rho}(a_{2:H}|s_1, \bar s_2, \bar r_1) + \epsilon\\
    &\leq \beta(a_1|s_1) \frac{P_\beta(g_1 \geq g_\rho|s_1, a_1)}{P_\beta(g_1 \geq g_\rho|s_1)}P_{\pi_\rho}(a_{2:H}|s_1, \bar s_2, \bar r_1) + \epsilon\\
    &\leq \beta(a_1|s_1) \frac{P_\beta(g_1 \geq g_\rho|\bar s_2, \bar r_1, a_1)}{P_\beta(g_1 \geq g_\rho|s_1)}P_{\pi_\rho}(a_{2:H}|s_1, \bar s_2, \bar r_1) + \epsilon \left(\frac{1}{\alpha_\rho} + 1\right)\\
    &\leq \beta(a_1|s_1)\beta(a_2|\bar s_2) \frac{\cancel{P_\beta(g_1 \geq g_\rho|\bar s_2, \bar r_1)}}{P_\beta(g_1 \geq g_\rho|s_1)}\frac{P_\beta(g_1 \geq g_\rho|\bar s_2, \bar r_1, a_2)}{\cancel{P_\beta(g_1 \geq g_\rho|\bar s_2, \bar r_1)}} P_{\pi_\rho}(a_{3:H}|s_1, \bar s_3, \bar r_{1:2})\\ &\qquad+ 2 \epsilon \left(\frac{1}{\alpha_\rho} + 1\right)\\
    &\leq \prod_{t=1}^H \beta(a_t|\bar s_t) \frac{P_\beta(g_1 \geq g_\rho|\bar s_H, \bar r_{1:H})}{P_\beta(g_1 \geq g_\rho|s_1)} + H \epsilon \left(\frac{1}{\alpha_\rho} + 1\right)\\
    &= \prod_{t=1}^H \beta(a_t|\bar s_t) \frac{\1[g(s_1, a_{1:H}) \geq g_\rho]}{P_\beta(g_1 \geq g_\rho|s_1)} + H \epsilon \left(\frac{1}{\alpha_\rho} + 1\right)
\end{align}
Plugging this into Equation \ref{eq:pbc-proof} we get the result.
\end{proof}

\begin{theorem}[Reduction of \%BC to SL]\label{thm:pbc-finite}
Let $ g_\rho $ be the $ 1-\rho$ percentile of the return distribution induced by $ \beta$. Let $ \pi_\rho = P_\beta(a|s, g \geq g_\rho)$. 
Assume
\begin{enumerate}
    \item Bounded mismatch: $\frac{P_{\pi_\rho}(s)}{P_\beta(s|g\geq g_\rho)} \leq C_\rho $ for all $ s$.
\end{enumerate}
Define the expected loss as $L_\rho(\hat \pi) = \E_{s \sim P_\beta|g\geq g_\rho}[KL(\pi_\rho(\cdot|s)\| \hat \pi(\cdot |s))]$. Then we have that 
\begin{align}
    J(\pi_\rho) - J(\hat \pi) \leq  C_\rho H^2 \sqrt{2 L_\rho(\hat \pi)}. 
\end{align}
\end{theorem}

\begin{proof}
Recall that $ d_\pi $ refers to the marginal distribution of $ P_\pi$ over states only.
Applying the definition of $J$ and Lemma \ref{lem:l1}, we get
\begin{align}
     J(\pi_\rho) - J(\hat \pi) &= H (\E_{P_{\pi_\rho}}[r(s,a)] - \E_{P_{\hat \pi}}[r(s,a)] )\\
     &\leq H \| d_{\pi_\rho} - d_{\hat \pi} \|_1\\
     &\leq 2 \cdot \E_{s \sim d_{\pi_\rho}}[TV(\pi_\rho(\cdot|s) \| \hat \pi(\cdot|s))] H^2
\end{align}
Expanding definitions, using the multiply and divide trick, and applying the assumptions:
\begin{align}
    2 \cdot \E_{s \sim d_{\pi_\rho}}[TV(\pi_\rho(\cdot|s) \| \hat \pi(\cdot|s))] 
    &\leq  C_\rho \cdot 2 \E_{s \sim P_\beta(\cdot|g \geq g_\rho)}[TV(\pi_\rho(\cdot|s) \| \hat \pi(\cdot|s))]\\
    &\leq C_\rho \sqrt{2 L(\hat \pi)}
\end{align}
where the last step comes from Pinsker's inequality.
Combining with the above bound on the difference in expected values yields the result. 
\end{proof}

\begin{corollary}[Sample complexity for \%BC]
To get finite data guarantees, add to the above assumptions the assumptions that (1) the policy class $ \Pi$ is finite, (2) $|\log \pi(a|s) - \log \pi(a'|s')| \leq c$ for any~$(a, s, a', s')$ and all~$\pi \in \Pi$, and (3) the approximation error of $ \Pi$ is bounded by $ \epsilon_{approx}$, i.e. $ \min_{\pi\in \Pi}L_\rho(\pi) \leq \epsilon_{approx}$. Then with probability at least $ 1-\delta$,
\begin{align}
    J(\pi_\rho) - J(\hat \pi) \leq O \left ( C_\rho H^2\left( \sqrt{c} \left(\frac{\log |\Pi| /\delta}{(1-\rho) N}\right)^{1/4} + \sqrt{\epsilon_{approx}} \right) + \frac{\epsilon}{\alpha_\rho} H^2\right). 
\end{align}
\end{corollary}

\section{Experimental details}\label{app:exp-details}

\paragraph{Data.} Data for point-mass tasks was sampled from the scripted policies described in the text. We sampled 100 trajectories of length 400 for each dataset, unless otherwise indicated. Data for the benchmark experiments was taken directly from the D4RL benchmark \citep{fu2020d4rl}.

\paragraph{Hyperparameters.}
Below we list all of the hyperparameters used across the various agorithms. We train each algorithm on 90\% of the trajectories in the dataset, using the remaining 10\% as validation. All algorithms are trained with the Adam optimizer \cite{kingma2014adam}. We evaluate each algorithm for 100 episodes in the environment per seed and hyperparameter configuration and report the best performance for each algorithm for it's relevant hyperparameter (All algorithms were tuned across 3 values of the hyperparameter except for DT on pointmass where we tried more values, but still got poor results). Error bars are reported across seeds, as explained in the text.

\begin{table}[h!]
    \caption{Shared hyperparameters for all non-DT algorithms}
    \centering
    \begin{tabular}{lc}
        \hline
        Hyperparameter       & Value \\
        \hline
        \: Training steps & $5e5$ \\
        \: Batch size & $256$ \\
        \: MLP width & $256$ \\
        \: MLP depth & $2$ \\
        \hline
    \end{tabular}
\end{table}

\begin{table}[h!]
    \caption{Algorithm-specific hyperparameters for all non-DT algorithms}
    \centering
    \begin{tabular}{lcc}
        \hline
        Algorithm & Hyperparameter  & Value(s) \\
        \hline
        \: \%BC & fraction  $ \rho$ &  [0.02, 0.10, 0.5]\\
        \:    & learning rate  & 1e-3\\\hline
        \: RvS & fraction of max return for conditioning &  [0.8, 1.0, 1.2]\\
        \:    & learning rate  & 1e-3\\\hline
        \: TD3+BC & $\alpha$ &  [1.0, 2.5, 5.0]\\
        \:      & learning rate (actor and critic)  & 3e-4\\
        \:      & discount  & 0.99\\
        \:      & $\tau$ for target EWMA  & 0.005\\
        \:      & target update period & 2\\\hline
        \: IQL & expectile &  [0.5, 0.7, 0.9]\\
        \:      & learning rate (actor, value, and critic)  & 3e-4\\
        \:      & discount  & 0.99\\
        \:      & $\tau$ for target EWMA  & 0.005\\
        \:      & temperature & 10.0\\
        \hline
    \end{tabular}
\end{table}

\begin{table}[h!]
    \caption{Hypereparameters for DT (exactly as in \citep{chen2021decision})}
    \centering
    \begin{tabular}{lc}
        \hline
        Hyperparameter       & Value \\
        \hline
        \: Training steps & $1e5$ \\
        \: Batch size & $64$ \\
        \: Learning rate & 1e-4 \\
        \: Weight decay & 1e-4 \\
        \: K & 20 \\
        \: Embed dimension & 128 \\
        \: Layers & 3 \\
        \: Heads & 1 \\
        \: Dropout & 0.1 \\
        \hline
    \end{tabular}
\end{table}

\begin{table}[h!]
    \caption{Environment-specific reward targets for DT}
    \centering
    \begin{tabular}{lc}
        \hline
        Environment       & Values \\
        \hline
        \: Point-mass & [300, 200, 100, 50, 10, 0] \\
        \: Antmaze & [1.0, 0.75, 0.5] \\
        \: Half-cheetah & [12000, 9000, 6000] \\
        \: Pen & [3000, 2000, 1000] \\
        \hline
    \end{tabular}
\end{table}

\paragraph{Compute.} All experiments were run on CPU on an internal cluster. Each of the non-DT algorithms takes less than 1 hour per run (i.e. set of hyperparameters and seed) and the DT algorithm takes 5-10 hours per run.  

\paragraph{Asset licenses.} For completeness, we also report the licenses of the assets that we used in the paper: JAX \cite{jax2018github}: Apache-2.0, Flax \cite{flax2020github}: Apache-2.0, jaxrl \cite{kostrikovjaxrl}: MIT, Decision Transformer \cite{chen2021decision}: MIT, Deepmind control suite \cite{tassa2018deepmind}: Apache-2.0, mujoco \cite{mujoco}: Apache-2.0, D4RL \cite{fu2020d4rl}: Apache-2.0. 

\paragraph{Code.} The code for our implementations can be found at \url{https://github.com/davidbrandfonbrener/rcsl-paper}.

\section{Potential negative societal impact}
\label{app:societal}
This paper follows  a line work aiming at a better understanding of Offline RL algorithms. Even though it does not  directly contribute to any specific application, it promotes the development and dissemination of the Offline RL technology, which, as any technology, can be used for harmful purposes. Moreover, we acknowledge that Offline RL has been proved in the past to lack robustness, and RL and even machine learning in general to potentially reproduce and amplify bias.

We note that this specific work attempts at better understanding the conditions for RCSL algorithms to work, and where it should not be used. In that spirit, it has the potential benefit of dissuading practitioners from using such algorithms in settings where. they may fail in socially undesirable ways.

\end{document}